\documentclass{article}

\usepackage{microtype}
\usepackage{graphicx}
\usepackage{subcaption}
\usepackage{booktabs} 
\usepackage{makecell} 
\usepackage{float}
\usepackage{hyperref}

\usepackage[preprint]{neurips_2026}

\usepackage[utf8]{inputenc}
\usepackage[T1]{fontenc}
\usepackage{url}
\usepackage{nicefrac}
\usepackage{xcolor}

\usepackage{amsmath}
\usepackage{amssymb}
\usepackage{mathtools}
\usepackage{amsthm}
\usepackage{bm}
\usepackage{tikz}
\usetikzlibrary{arrows.meta,matrix,fit,backgrounds,shadows,calc}

\usepackage{algorithm}
\usepackage{algorithmic}
\usepackage{multirow}
\usepackage[capitalize,noabbrev]{cleveref}
\usepackage{longtable}
\usetikzlibrary{shadows}
\theoremstyle{plain}
\newtheorem{proposition}{Proposition}

\newtheorem{theorem}{Theorem}
\theoremstyle{definition}
\newtheorem{remark}{Remark}

\usepackage[textsize=tiny]{todonotes}

\title{SOC-ICNN: From Polyhedral to Conic Geometry for Learning Convex Surrogate Functions}

\author{%
  Kang Liu\\
  School of Future Technology\\
  Xi'an Jiaotong University\\
  Xi'an, China\\
  \texttt{kanyo@foxmail.com}\\
  \And
  Jianchen Hu$^*$\\
  School of Future Technology\\
  School of Automation Science and Engineering\\
  Xi'an Jiaotong University\\
  Xi'an, China\\
  \texttt{horace89@gmail.com}\\
  \And
  Wei Peng \\
  School of Automation Science and Engineering \\
  Xi'an Jiaotong University \\
  Xi'an, China \\
  \texttt{pengwei@stu.xjtu.edu.cn}\\
}

\begin{document}

\maketitle

\begin{abstract}
Classical ReLU-based Input Convex Neural Networks (ICNNs) are equivalent to the optimal value functions of Linear Programming (LP). This intrinsic structural equivalence restricts their representational capacity to piecewise-linear polyhedral functions. To overcome this representational bottleneck, we propose the SOC-ICNN, an architecture that generalizes the underlying optimization class from LP to Second-Order Cone Programming (SOCP). By explicitly injecting positive semi-definite curvature and Euclidean norm-based conic primitives, our formulation introduces native smooth curvature into the representation while preserving a rigorous optimization-theoretic interpretation. We formally prove that SOC-ICNNs strictly expand the representational space of ReLU-ICNNs without increasing the asymptotic order of forward-pass complexity. Extensive experiments demonstrate that SOC-ICNN substantially improves function approximation, while delivering competitive downstream decision quality. The code is available at \url{https://anonymous.4open.science/r/SOC-ICNN-4B18/}.
\end{abstract}

\section{Introduction}
\label{sec:intro}

Input Convex Neural Networks (ICNNs) rigorously guarantee the convexity of their outputs with respect to specified inputs. Unlike generic black-box models, a convex surrogate enables tractable and exact global optimization over the decision variables, which is a critical advantage in downstream applications such as structured prediction, inverse optimization, optimal transport, and parametric decision-making \citep{AmosXuKolter2017ICNN,MakkuvaTaghvaeiOhLee2020OTICNN,RosembergTanneauFanzeresEtAl2024OPF}.

A standard paradigm is to learn a parameterized convex surrogate \(f_\Theta(\bm{x})\) and subsequently solve the downstream task:
\begin{equation}
\min_{\bm{x}\in\mathcal{X}} f_\Theta(\bm{x}),
\label{eq:intro_problem_form}
\end{equation}
where \(\mathcal{X}\) is a feasible set and \(\Theta\) denotes the network parameters. For the widely used ReLU-ICNN, while prior literature has observed connections between its forward pass and linear programming (LP) \citep{AmosXuKolter2017ICNN}, we take a step further to rigorously formalize finite-width ReLU-ICNNs as exact LP value-function networks. In this view, the network output identically equals the optimal value of a parameterized LP. This perspective clarifies both their structural strength (the natural representation of continuous piecewise-linear (CPWL) convex polyhedra) and their fundamental intrinsic limitation: they cannot natively encode smooth curvature or non-polyhedral conic geometries, such as Euclidean norms. More critically, this CPWL confinement implies a severe parameter inefficiency: as we later quantify, approximating even a simple quadratic function to high accuracy forces a ReLU-ICNN to approximate the input space with an exponentially growing number of linear pieces. Recent theoretical work further substantiates this by proving lower bounds on the depth complexity of ICNNs and characterizing the exact conditions under which a ReLU network can be convex \citep{gagneux2025convexity,bakaev2025depth}.

Recent research has sought to improve ICNN expressivity predominantly along two lines. The first line enriches local nonlinearities, for instance, by employing smoother activations (e.g., Softplus) or learnable univariate basis functions, as seen in Input Convex Kolmogorov-Arnold Networks (ICKAN) \citep{DeschatreWarin2025ICKAN}. The second line focuses on improving trainability through tailored initialization schemes or architectural adaptations \citep{HoedtKlambauer2023ICNNInit}. While these approaches enhance empirical flexibility and training stability, they incur a fundamental trade-off: by departing from the ReLU-based piecewise-linear structure, they sacrifice the exact value-function interpretation that gives ReLU-ICNNs their transparent optimization-theoretic grounding. Moreover, they remain strictly confined within approximation frameworks that cannot exactly represent non-polyhedral convex sets like Euclidean norm balls.

Another relevant trajectory builds upon differentiable convex optimization layers, where a convex program is explicitly embedded as a neural network layer and differentiated via its solution map \citep{AgrawalAmosBarrattEtAl2019DiffCvx,BesanconGarciaLegatEtAl2024DiffOpt}. More recently, Learning Parametric Convex Functions (PCF) has been proposed to fit disciplined-programming-compatible convex families directly from data \citep{SchallerBemporadBoyd2025LPCF}. While these methods successfully output optimal decisions for given parameters, they intrinsically rely on solver-in-the-loop routines during inference, thereby incurring substantially higher computational overhead than direct feed-forward evaluation. Furthermore, they often require the optimization template or disciplined structure to be specified \emph{a priori}.

In this paper, we pursue a fundamentally different structural route that overcomes the limitations of both local-nonlinearity enhancements and solver-dependent approaches. Since ReLU-ICNNs are strictly LP value-function networks, a natural question arises: \emph{Can we elevate the architecture to a richer convex optimization class while strictly preserving both the transparent optimization interpretation and rapid forward inference?} Among mathematical programming extensions beyond LP, Second-Order Cone Programming (SOCP) emerges as the most rigorous and natural candidate. It elegantly unifies quadratic curvature and norm-based geometry, which together cover a vast class of smooth and structured convex objectives. Motivated by this, we propose the SOC-ICNN, which augments the polyhedral ReLU backbone with two explicit geometric primitives:
(1) \textbf{quadratic branch} to inject positive semi-definite (PSD) curvature: \(\frac{\alpha}{2}\|B\bm{x}+\bm{e}\|_2^2\).
(2) \textbf{conic branch} to capture Euclidean norm geometry: \(\lambda\|A\bm{x}+\bm{d}\|_2\).

Because both primitives are convex, their non-negative linear combination strictly preserves convexity. Crucially, each primitive admits a SOCP epigraph lift. Consequently, the entire forward pass of the network can be mathematically formulated as the optimal value of a finite-dimensional SOCP. The proposed SOC-ICNN therefore constitutes an architectural upgrade from LP to SOCP, simultaneously breaking the parameter-inefficient CPWL bottleneck and retaining a rigorous optimization-theoretic interpretation. The primary contributions of this paper are twofold:
\begin{enumerate}
    \item \textbf{Architectural generalization:} We propose SOC-ICNN, systematically elevating the underlying mathematical structure of ICNNs from LP to SOCP value functions by explicitly injecting curvature and conic geometric primitives, with the same lifting principle extending naturally to convolutional and finite-horizon recurrent architectures.
    \item \textbf{Theoretical guarantees:} We formally prove that the SOC-ICNN strictly expands the representational function class of classical ReLU-ICNNs, doing so without increasing the asymptotic order of forward-pass computational complexity.
\end{enumerate}

The remainder of this paper is organized as follows. 
Section~\ref{sec:related} reviews related work. 
Section~\ref{sec:method} details the architectural design of SOC-ICNN. 
Section~\ref{sec:theory} provides a rigorous theoretical analysis. 
Section~\ref{sec:exp} presents empirical evaluations, followed by conclusions in Section~\ref{sec:conclusion}. 
The appendix provides complete proofs, full experimental results, and additional convolutional and finite-horizon recurrent SOC-ICNN extensions with downstream tests.

\section{Related Work}
\label{sec:related}

\paragraph{ICNNs and their expressive extensions.}
ICNNs were introduced in \citep{AmosXuKolter2017ICNN} as neural architectures with built-in input convexity, including both fully input-convex and partially input-convex constructions, thereby enabling optimization-aware prediction and inference.
Since then, ICNNs have been used as structured convex surrogates in applications such as optimal transport \citep{MakkuvaTaghvaeiOhLee2020OTICNN}, value-function learning for optimal power flow \citep{RosembergTanneauFanzeresEtAl2024OPF}, two-stage stochastic programming \citep{LiuOliveiraKronqvist2025ICNN2SP}, and power system contingency screening with provable reliability guarantees \citep{christianson2025fast}.

A central limitation of finite-width ReLU-ICNNs, however, is that their representable class is tied to polyhedral, piecewise-linear convex functions.
Existing work has mainly addressed this limitation by enriching the parameterization or extending the backbone within the ICNN framework.
One line adapts ICNNs to sequential and control-oriented settings, including input-convex recurrent models for optimal control \citep{ChenShiZhang2019ICRNN}, input-convex LSTM architectures for real-time optimization \citep{WangYuWu2023ICLSTM}, and recurrent variants that further incorporate Lipschitz constraints for robustness and efficient process modeling \citep{wang2024input}.
Another line enriches the internal parameterization of input-convex models.
In particular, ICKAN \citep{DeschatreWarin2025ICKAN} introduces learnable univariate basis functions into input-convex architectures, improving empirical flexibility beyond standard ReLU-based constructions.
Concurrently, theoretical investigations have characterized the necessary and sufficient conditions under which a ReLU network is convex \citep{gagneux2025convexity}, and have established lower bounds on the depth complexity of ICNNs, highlighting fundamental expressivity limitations \citep{bakaev2025depth}.

While these approaches enlarge the practical approximation power of ICNNs, they do not explicitly upgrade the underlying optimization-theoretic structure beyond the LP/polyhedral regime.
In contrast, our goal is not merely to enrich the local nonlinearity or change the network backbone, but to lift the value-function class itself, moving from LP value functions to SOCP value functions.

\paragraph{Parametric convex models and differentiable convex programs.}
Another related line of work builds optimization-aware models by embedding or learning convex programs directly.
Early work such as OptNet \citep{AmosKolter2017OptNet} introduced differentiable optimization layers based on quadratic programs.
Differentiable convex optimization layers then generalized this idea to disciplined convex programs and cone programs, including SOCPs \citep{AgrawalAmosBarrattEtAl2019DiffCvx}, and DiffOpt \citep{BesanconGarciaLegatEtAl2024DiffOpt} further develops this perspective through differentiable optimization via model transformations.
More recently, BPQP \citep{pan2024bpqp} introduces a differentiable convex optimization framework that reformulates the backward pass as a decoupled quadratic program for enhanced efficiency, and comprehensive surveys have synthesized developments at the intersection of optimization theory and deep learning \citep{katyal2024differentiable}.
Learning Parametric Convex Functions (PCF) has been proposed to fit disciplined-programming-compatible convex families directly from data \citep{SchallerBemporadBoyd2025LPCF}, and very recent work explores constructing convex surrogates for parametric nonconvex optimization via compositions of convex and monotonic terms \citep{wang2026parametric}.

These methods typically begin from a prescribed optimization template and rely on differentiating through the corresponding solution map or learning a disciplined-programming-compatible expression tied to a solver layer.  By contrast, we construct a direct feed-forward architecture whose forward pass itself admits an exact SOCP value-function interpretation, thereby lifting ICNNs from polyhedral LP geometry to conic SOCP geometry while retaining fast inference.

\section{Method}
\label{sec:method}

This section introduces the proposed SOC-ICNN by first formalizing ReLU-ICNNs as LP value functions. We then rigorously expose the fundamental representational bottleneck of this LP class so that we elevate the backbone to SOCP and finally prove its exact SOCP formulation.

\subsection{ReLU-ICNN as an LP value-function backbone}
\label{subsec:lp_backbone}

Given an input \(\bm{x}\in\mathbb{R}^{d_0}\), a standard ReLU-ICNN is defined by
\(\bm{z}_\ell=\sigma(W_\ell\bm{x}+U_\ell\bm{z}_{\ell-1}+\bm{b}_\ell)\), \(\ell=1,\dots,L\), with \(\bm{z}_0=\bm{0}\), elementwise ReLU \(\sigma(t)=\max\{t,0\}\), and nonnegativity constraints \(U_\ell\ge0\), \(\bm{c}\ge0\). 
Its output is \(f_{\mathrm{ReLU}}(\bm{x})=\bm{c}^\top\bm{z}_L+\bm{v}^\top\bm{x}+b_0\). 
We first make explicit a structural fact behind this architecture: its forward pass is exactly an LP value function.

\begin{proposition}[LP value-function representation]
\label{prop:relu_lp_method}
For any input \(\bm{x}\), \(f_{\mathrm{ReLU}}(\bm{x})\) is exactly equal to the optimal value of
\begin{equation}
\begin{aligned}
&f_{\mathrm{ReLU}}(\bm{x})
=
\min_{\{\bm{z}_\ell\}_{\ell=1}^L}\quad
\bm{c}^\top \bm{z}_L+\bm{v}^\top \bm{x}+b_0 \\
{\rm s.t.}\quad
&\bm{z}_\ell \ge W_\ell \bm{x}+U_\ell \bm{z}_{\ell-1}+\bm{b}_\ell,\quad
\bm{z}_0=\bm{0},\quad
\bm{z}_\ell\ge\bm{0},\quad \ell=1,\dots,L .
\end{aligned}
\label{eq:method_relu_lp}
\end{equation}
\end{proposition}

\begin{proof}[Proof sketch]
The ReLU recursion constructs the componentwise minimal feasible hidden states because \(\sigma(a)=\max\{a,0\}\) elementwise. 
For any feasible \(\{\bm z_\ell\}\), the monotonicity induced by \(U_\ell\ge0\) gives \(\bm z_\ell\ge\bar{\bm z}_\ell\) by induction, where \(\bar{\bm z}_\ell\) are the forward activations. 
Since \(\bm c\ge0\), every feasible point has objective value at least the forward value, while the forward activations themselves are feasible. 
Thus the LP optimum equals \(f_{\mathrm{ReLU}}(\bm{x})\). Details are given in Appendix~\ref{app:relu_lp_equiv_proof}.
\end{proof}

Proposition~\ref{prop:relu_lp_method} exposes the exact optimization class realized by finite-width ReLU-ICNNs: they are LP value-function models. 
This observation is not merely interpretive; it identifies the source of both their strength and their bottleneck. 
To break the bottleneck while preserving a value-function interpretation, the natural route is to enlarge the underlying optimization class rather than only modify local nonlinearities.

\subsection{Fundamental Limitations of LP Value Functions}
\label{subsec:lp_limitations}

Because the ReLU-ICNN forward pass is the value of the parametric LP in \eqref{eq:method_relu_lp}, its output is confined to the CPWL class. 
The following proposition makes this restriction explicit.

\begin{proposition}[CPWL structure]
\label{prop:relu_cpwl_theory}
Any finite-depth, finite-width ReLU-ICNN represents a function \(f_{\mathrm{ReLU}}:\mathbb{R}^{d_0}\to\mathbb{R}\) that can be written exactly as a finite max-affine convex function
\[
f_{\mathrm{ReLU}}(\bm{x})
=
\max_{j\in\mathcal{J}}
\{\bm{a}_j^\top\bm{x}+\beta_j\},
\]
where \(\mathcal{J}\) is finite. 
In particular, \(f_{\mathrm{ReLU}}\) is continuous, convex, and piecewise linear.
\end{proposition}

\begin{proof}[Proof sketch]
Dualizing \eqref{eq:method_relu_lp} yields a linear maximization problem over a bounded polyhedron, with boundedness induced by the nonnegative dual chain \(0\le\bm\nu_L\le\bm c\) and \(0\le\bm\nu_\ell\le U_{\ell+1}^\top\bm\nu_{\ell+1}\). 
A linear objective over a polyhedron attains its maximum at an extreme point, and finitely many extreme points yield a finite max-affine representation. 
Details are provided in Appendix~\ref{app:relu_cpwl_proof}.
\end{proof}

\paragraph{Regional flat limitation.}
Proposition~\ref{prop:relu_cpwl_theory} implies a rigid polyhedral subgradient structure. 
For any input \(\bm{x}\) and any \(\bm{g}\in\partial f_{\mathrm{ReLU}}(\bm{x})\),
\begin{equation}
\bm{g}
=
\bm{v}
+
\sum_{\ell=1}^{L}
W_\ell^\top \bm{\nu}_\ell,
\label{eq:relu_subgrad_cone}
\end{equation}
where \(\{\bm{\nu}_\ell\}\) are confined to the polyhedral constraints in \eqref{eq:app_dual_chain}. 
Thus the gradient mapping \(\bm{x}\mapsto\partial f_{\mathrm{ReLU}}(\bm{x})\) is piecewise constant: the Jacobian is zero almost everywhere, and the model has no native mechanism for smooth curvature on each linear region.

This regional flatness becomes a severe quantitative bottleneck when the target has genuine curvature. 
The next result shows that CPWL models must spend a large number of affine pieces merely to approximate strongly convex geometry.

\begin{proposition}[CPWL lower bound for strongly convex targets]
\label{prop:relu_lower_bound}
Let \(\Omega\subset\mathbb{R}^{d_0}\) be compact and convex with nonzero volume, and let \(f:\Omega\to\mathbb{R}\) be \(\mu\)-strongly convex on \(\Omega\).
If a CPWL function \(g(\bm{x})=\max_{1\le i\le N}\{\bm a_i^\top\bm x+b_i\}\) composed of \(N\) affine pieces satisfies \(\|f-g\|_{L_\infty(\Omega)}\le\varepsilon\), then
\begin{equation}
N
\ge
\frac{\operatorname{vol}_{d_0}(\Omega)}
{\omega_{d_0}\,2^{d_0}}
\left(\frac{\mu}{\varepsilon}\right)^{d_0/2},
\label{eq:piece_lower_bound}
\end{equation}
where \(\omega_{d_0}=\operatorname{vol}_{d_0}(B_2^{d_0})\).
\end{proposition}

\begin{proof}[Proof sketch]
Let \(A_i=\{\bm{x}\in\Omega:g(\bm{x})=\bm a_i^\top\bm x+b_i\}\) be the active region of the \(i\)-th affine piece and refine these sets into a measurable partition \(P_i\). 
For any \(\bm{x},\bm{y}\in A_i\), strong convexity at the midpoint together with \(\|f-g\|_{L_\infty(\Omega)}\le\varepsilon\) yields \(\|\bm{x}-\bm{y}\|_2\le4\sqrt{\varepsilon/\mu}\). 
Hence each \(P_i\) has diameter at most \(4\sqrt{\varepsilon/\mu}\). 
The isodiametric inequality gives \(\operatorname{vol}(P_i)\le \omega_{d_0}(2\sqrt{\varepsilon/\mu})^{d_0}\). 
Summing over all pieces and rearranging yields \eqref{eq:piece_lower_bound}. 
A complete proof is provided in Appendix~\ref{app:relu_lower_bound_proof}.
\end{proof}

This lower bound formalizes the curse of polyhedral approximation: approximating smooth strongly convex geometry to accuracy \(\varepsilon\) requires \(N=\Omega(\varepsilon^{-d_0/2})\) affine pieces. 
For ReLU-ICNNs, this means that curvature must be simulated indirectly by proliferating linear regions. 
This is precisely the structural bottleneck that SOC-ICNN removes. 
Instead of spending polyhedral pieces to imitate curvature, we lift the value-function class from LP to SOCP and inject quadratic and conic primitives directly into the architecture.

\subsection{SOC-ICNN Architecture}
\label{subsec:soc_icnn_architecture}

The fundamental limitation of LP value functions is their purely polyhedral geometry. To overcome this, we augment the standard ReLU-ICNN backbone with two explicit structural primitives. From a neural network perspective, this results in a topology where the input $\bm{x}$ is processed through three parallel computational branches before being aggregated. The same construction also extends to convolutional and finite-horizon recurrent architectures by replacing dense affine maps with structured convolutional or unrolled recurrent linear operators; see Appendices~\ref{app:cnn_extension} and~\ref{app:rnn_extension}.

\paragraph{Quadratic primitive.}
For each $h=1,\dots,H$, we introduce a quadratic branch defined as
\begin{equation}
\frac{\alpha_h}{2}\left\|B_h\bm{x}+\bm{e}_h\right\|_2^2,
\qquad
\alpha_h\ge 0,
\label{eq:quad_module}
\end{equation}
where $B_h\in\mathbb{R}^{m_h\times d_0}$ and $\bm{e}_h\in\mathbb{R}^{m_h}$. Architecturally, this branch applies a learnable affine transformation to the input, routes the resulting vector through a squared $\ell_2$-norm pooling operator, and scales it by a non-negative weight $\alpha_h$, which provides an explicit PSD curvature primitive.

\paragraph{Conic primitive.}
For each $g=1,\dots,G$, we introduce a conic branch defined as:
\begin{equation}
\lambda_g\left\|A_g\bm{x}+\bm{d}_g\right\|_2,
\qquad
\lambda_g\ge 0,
\label{eq:soc_module}
\end{equation}
where $A_g\in\mathbb{R}^{k_g\times d_0}$ and $\bm{d}_g\in\mathbb{R}^{k_g}$. Similarly, this branch projects the input via an affine layer, passes it through an un-squared standard $\ell_2$-norm operator, and scales the output by $\lambda_g$. This term acts as an explicit Euclidean norm primitive.

Combining the outputs of these two parallel branches with the main ReLU backbone yields the unified SOC-ICNN forward pass:
\begin{equation}
\begin{aligned}
f_{\mathrm{SOC}}(\bm{x})
=
f_{\mathrm{ReLU}}(\bm{x})
+
\sum_{h=1}^{H}
\frac{\alpha_h}{2}
\left\| B_h \bm{x}+\bm{e}_h \right\|_2^2+
\sum_{g=1}^{G}
\lambda_g
\left\| A_g \bm{x}+\bm{d}_g \right\|_2.
\label{eq:unified_soc_icnn}
\end{aligned}
\end{equation}

The three terms in \eqref{eq:unified_soc_icnn} play distinct and complementary geometric roles. The ReLU backbone captures the polyhedral (piecewise-linear) structure, the quadratic branch injects explicit smooth curvature, and the norm branch captures second-order conic geometry. Thus, the SOC-ICNN extends classical ReLU-ICNNs from purely polyhedral value functions into a unified class of structured convex functions encompassing piecewise-linear, quadratic, and conic components. The overall parallel topology is summarized in Figure \ref{fig:unified_arch}.

\begin{figure}[ht]
\centering
\resizebox{1.0\linewidth}{!}{
\begin{tikzpicture}[
    >=Stealth,
    op/.style={
        draw, thick,
        rounded corners=3pt,
        minimum height=0.82cm,
        align=center,
        font=\scriptsize,
        inner sep=4pt,
        drop shadow={opacity=0.08, shadow xshift=0.03cm, shadow yshift=-0.04cm}
    },
    quad/.style={op, fill=teal!4, draw=teal!70!black},
    relu/.style={op, fill=orange!4, draw=orange!80!black},
    conic/.style={op, fill=purple!4, draw=purple!70!black},
    sum/.style={
        draw=black!60,
        thick,
        rounded corners=2pt,
        rectangle,
        fill=gray!8,
        minimum height=0.58cm,
        minimum width=0.68cm,
        inner sep=2pt,
        font=\footnotesize\bfseries,
        drop shadow={opacity=0.08, shadow xshift=0.03cm, shadow yshift=-0.04cm}
    },
    junction/.style={
        draw=black!55,
        fill=gray!10,
        thick,
        rectangle,
        minimum size=0.13cm,
        inner sep=0pt
    },
    line/.style={-{Stealth[length=1.8mm, width=1.35mm]}, thick, draw=black!60},
    plainline/.style={thick, draw=black!60},
    branchlabel/.style={
        font=\scriptsize\bfseries,
        anchor=south west,
        inner sep=1.5pt
    }
]

\def\yq{1.75}
\def\yr{0}
\def\yc{-1.75}

\def\xinput{0}
\def\xsplit{0.8}
\def\xone{2.8}
\def\xtwo{6.1}
\def\xthree{9.4}
\def\xfour{13.1}
\def\xsum{16.5}
\def\xout{17.9}

\node[font=\normalsize\bfseries] (x) at (\xinput,\yr) {$\bm{x}$};

\node[quad, minimum width=3.0cm] (q1) at (\xone,\yq) 
{Affine map\\$\bm{q}_h=B_h\bm{x}+\bm{e}_h$};

\node[quad, minimum width=2.9cm] (q2) at (\xtwo,\yq) 
{Squared norm\\$\frac12\|\bm{q}_h\|_2^2$};

\node[quad, minimum width=2.7cm] (q3) at (\xthree,\yq) 
{Scaling\\$\alpha_h\ge0$};

\node[quad, minimum width=4.1cm] (q4) at (\xfour,\yq) 
{Output\\$\sum_{h=1}^{H}\frac{\alpha_h}{2}\|\bm{q}_h\|_2^2$};

\node[relu, minimum width=3.4cm] (r1) at (\xone,\yr) 
{Hidden update\\$W_\ell\bm{x}+U_\ell\bm{z}_{\ell-1}+\bm{b}_\ell$};

\node[relu, minimum width=2.8cm] (r2) at (\xtwo,\yr) 
{Activation\\$\bm{z}_\ell=\sigma(\cdot)$};

\node[relu, minimum width=2.7cm] (r3) at (\xthree,\yr) 
{Convexity\\$U_\ell\ge0,\ \bm{c}\ge0$};

\node[relu, minimum width=4.1cm] (r4) at (\xfour,\yr) 
{Output\\$\bm{c}^\top\bm{z}_L+\bm{v}^\top\bm{x}+b_0$};

\node[conic, minimum width=3.0cm] (c1) at (\xone,\yc) 
{Affine map\\$\bm{u}_g=A_g\bm{x}+\bm{d}_g$};

\node[conic, minimum width=2.9cm] (c2) at (\xtwo,\yc) 
{Norm primitive\\$\|\bm{u}_g\|_2$};

\node[conic, minimum width=2.7cm] (c3) at (\xthree,\yc) 
{Scaling\\$\lambda_g\ge0$};

\node[conic, minimum width=4.1cm] (c4) at (\xfour,\yc) 
{Output\\$\sum_{g=1}^{G}\lambda_g\|\bm{u}_g\|_2$};

\node[sum] (s) at (\xsum,\yr) {$\Sigma$};
\node[font=\normalsize\bfseries] (out) at (\xout,\yr) {$f_{\mathrm{SOC}}(\bm{x})$};

\node[branchlabel, text=teal!80!black] (q_label) at ([yshift=0.14cm]q1.north west) {Quadratic Branch};
\node[branchlabel, text=orange!80!black] (r_label) at ([yshift=0.14cm]r1.north west) {ReLU-ICNN Backbone};
\node[branchlabel, text=purple!80!black] (c_label) at ([yshift=0.14cm]c1.north west) {Conic Branch};

\begin{scope}[on background layer]
    \node[
        draw=teal!40, dashed, thick, rounded corners=4pt, fill=teal!2, 
        fit=(q_label) (q1) (q4), inner xsep=0.22cm, inner ysep=0.18cm
    ] {};
    \node[
        draw=orange!40, dashed, thick, rounded corners=4pt, fill=orange!2, 
        fit=(r_label) (r1) (r4), inner xsep=0.22cm, inner ysep=0.18cm
    ] {};
    \node[
        draw=purple!40, dashed, thick, rounded corners=4pt, fill=purple!2, 
        fit=(c_label) (c1) (c4), inner xsep=0.22cm, inner ysep=0.18cm
    ] {};
\end{scope}

\node[junction] (split) at (\xsplit,\yr) {};
\draw[plainline] (x.east) -- (split.west);
\draw[line] (split.north) |- (q1.west);
\draw[line] (split.east) -- (r1.west);
\draw[line] (split.south) |- (c1.west);

\draw[line] (q1.east) -- (q2.west);
\draw[line] (q2.east) -- (q3.west);
\draw[line] (q3.east) -- (q4.west);

\draw[line] (r1.east) -- (r2.west);
\draw[line] (r2.east) -- (r3.west);
\draw[line] (r3.east) -- (r4.west);

\draw[line] (c1.east) -- (c2.west);
\draw[line] (c2.east) -- (c3.west);
\draw[line] (c3.east) -- (c4.west);

\node[junction] (vtop) at ($(s.west)+(-0.75,0.95)$) {};
\node[junction] (vmid) at ($(s.west)+(-0.75,0)$) {};
\node[junction] (vbot) at ($(s.west)+(-0.75,-0.95)$) {};

\draw[line] (q4.east) -- ($(vtop |- q4.east)$) -- (vtop.north);
\draw[plainline] (vtop.south) -- (vmid.north);
\draw[line] (vmid.east) -- (s.west);

\draw[line] (r4.east) -- (vmid.west);

\draw[line] (c4.east) -- ($(vbot |- c4.east)$) -- (vbot.south);
\draw[plainline] (vbot.north) -- (vmid.south);

\draw[line] (s.east) -- (out.west);

\end{tikzpicture}
}
\captionsetup{font=scriptsize}
\caption{Unified SOC-ICNN architecture. The input $\bm{x}$ is processed by three parallel branches. The quadratic branch applies an affine map followed by a squared $\ell_2$ primitive and nonnegative scaling, the ReLU-ICNN backbone captures polyhedral piecewise-linear structure under standard ICNN convexity constraints, and the conic branch applies an affine map followed by an $\ell_2$-norm primitive and nonnegative scaling. Summing the three branch outputs yields the final convex surrogate $f_{\mathrm{SOC}}(\bm{x})$, which admits an exact SOCP value-function interpretation.}
\label{fig:unified_arch}
\end{figure}
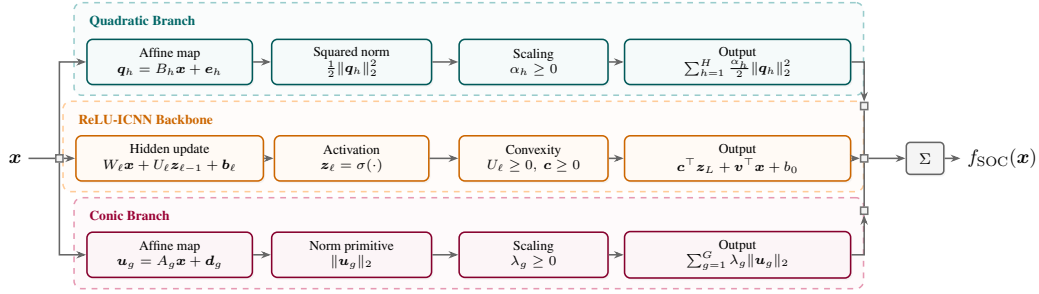

\subsection{SOCP lifting and value-function interpretation}
\label{subsec:soc_value_function}

Both structural branches admit standard SOCP epigraph lifts. 
For the quadratic branch, let
\(\mathcal Q_r^{m+2}:=\{(\xi,\eta,\bm w)\in\mathbb R_+\times\mathbb R_+\times\mathbb R^m:2\xi\eta\ge\|\bm w\|_2^2\}\)
denote the rotated second-order cone. 
Introducing \(\bm q_h=B_h\bm x+\bm e_h\) and \((s_h,1,\bm q_h)\in\mathcal Q_r^{m_h+2}\) gives
\(s_h\ge\frac12\|B_h\bm x+\bm e_h\|_2^2\), which is tight at optimum because \(\alpha_h\ge0\).
For the norm branch, let
\(\mathcal Q^{k+1}:=\{(\bm u,t)\in\mathbb R^k\times\mathbb R:\|\bm u\|_2\le t\}\)
denote the standard second-order cone. 
Introducing \(\bm u_g=A_g\bm x+\bm d_g\) and \((\bm u_g,t_g)\in\mathcal Q^{k_g+1}\) gives
\(t_g\ge\|A_g\bm x+\bm d_g\|_2\), again tight at optimum because \(\lambda_g\ge0\). We now state the main structural result: the entire SOC-ICNN forward pass is exactly an SOCP value function.

\begin{theorem}[SOC-ICNN as an SOCP value function]
\label{thm:soc_value_function}
For every input \(\bm{x}\), the unified architecture in \eqref{eq:unified_soc_icnn} is equal to the optimal value of the following SOCP:
\begin{equation}
\begin{aligned}
f_{\mathrm{SOC}}(\bm{x})
=
\min_{\eta}\quad
&\bm{c}^\top \bm{z}_L+\bm{v}^\top \bm{x}+b_0
+ \sum_{h=1}^{H}\alpha_h s_h
+ \sum_{g=1}^{G}\lambda_g t_g \\
{\rm s.t.}\quad
&\bm{z}_\ell \ge W_\ell \bm{x}+U_\ell \bm{z}_{\ell-1}+\bm{b}_\ell,\quad
\bm{z}_0=\bm{0},\quad
\bm{z}_\ell\ge\bm{0},\quad \ell=1,\dots,L, \\
&\bm{q}_h = B_h\bm{x}+\bm{e}_h,\quad
(s_h,1,\bm{q}_h)\in\mathcal{Q}_r^{m_h+2},\quad h=1,\dots,H, \\
&\bm{u}_g = A_g\bm{x}+\bm{d}_g,\quad
(\bm{u}_g,t_g)\in\mathcal{Q}^{k_g+1},\quad g=1,\dots,G,
\end{aligned}
\label{eq:soc_icnn_socp}
\end{equation}
where \(\eta=\{\bm{z}_\ell,s_h,\bm{q}_h,t_g,\bm{u}_g\}\).
\end{theorem}

\begin{proof}[Proof sketch]
Let \(V(\bm{x})\) denote the optimal value of \eqref{eq:soc_icnn_socp}. 
For the upper bound, choose the ReLU forward activations, set
\(\bm q_h=B_h\bm x+\bm e_h\), \(s_h=\frac12\|\bm q_h\|_2^2\), \(\bm u_g=A_g\bm x+\bm d_g\), and \(t_g=\|\bm u_g\|_2\). 
This point is feasible and attains exactly the value in \eqref{eq:unified_soc_icnn}, hence \(V(\bm{x})\le f_{\mathrm{SOC}}(\bm{x})\). 
For the lower bound, any feasible point satisfies the ReLU LP lower bound from Proposition~\ref{prop:relu_lp_method} and the conic epigraph inequalities
\(s_h\ge\frac12\|B_h\bm x+\bm e_h\|_2^2\), \(t_g\ge\|A_g\bm x+\bm d_g\|_2\). 
Multiplying by \(\alpha_h,\lambda_g\ge0\) and summing gives \(V(\bm{x})\ge f_{\mathrm{SOC}}(\bm{x})\). 
Thus \(V(\bm{x})=f_{\mathrm{SOC}}(\bm{x})\). A complete derivation is given in Appendix~\ref{app:soc_value_function_proof}.
\end{proof}

Since the ReLU backbone is convex in \(\bm x\), squared Euclidean norms and Euclidean norms composed with affine maps are convex, and nonnegative linear combinations preserve convexity, \(f_{\mathrm{SOC}}\) remains globally convex with respect to the input.

\begin{remark}[Passthrough structures and gradient richness]
\label{rem:passthrough}
In ReLU-ICNNs, the nonnegative hidden weights \(U_\ell\ge0\) restrict the propagated hidden-layer slopes, so unconstrained passthrough terms \(W_\ell\bm x\) are essential for expressing negative directions. 
SOC-ICNN adds a stronger source of gradient richness: its structural affine maps \(B_h\) and \(A_g\) are unconstrained, and the quadratic branch contributes gradients of the form \(\alpha_h B_h^\top(B_h\bm x+\bm e_h)\). 
Thus the model does not rely solely on ReLU passthrough connections to generate signed directional behavior; curvature and conic geometry enter through unconstrained structural branches.
\end{remark}

\section{Theoretical Analysis}
\label{sec:theory}

This section establishes the theoretical core of SOC-ICNN. 
After exposing the CPWL bottleneck of ReLU-ICNNs in Section~\ref{subsec:lp_limitations}, we show that SOC-ICNN breaks this bottleneck by lifting the underlying value-function class. 
The analysis proceeds in three steps: exact structural representation, structural absorption for residual approximation, and forward-complexity preservation.

\subsection{Exact Structured Representations and Strict Extension}
\label{subsec:strict_extension}

Having identified the CPWL bottleneck, we now show that the quadratic and conic branches do not merely improve empirical flexibility; they strictly enlarge the representable convex function class. 
Define the structured conic family
\[
\mathcal{G}_{\mathrm{SOC}}(r,G)
=
\bigl\{
\bm{a}^\top\bm{x}+b+\tfrac12\bm{x}^\top Q\bm{x}
+\sum_{g=1}^{G}\lambda_g\|A_g\bm{x}+\bm{d}_g\|_2\mid Q\succeq0,\ \operatorname{rank}(Q)\le r,\ \lambda_g\ge0
\bigr\}.
\]
This class contains affine terms, PSD quadratic curvature, and Euclidean norm primitives.

\begin{proposition}[Exact representation and strict extension]
\label{prop:soc_extension}
For any target function \(f\in \mathcal{G}_{\mathrm{SOC}}(r,G)\), there exists a finite-width SOC-ICNN such that
\[
f_{\mathrm{SOC}}(\bm{x}) \equiv f(\bm{x}),
\qquad
\forall \bm{x}\in\mathbb{R}^{d_0}.
\]
Moreover, if \(\mathcal{F}_{\mathrm{ReLU}}\) and \(\mathcal{F}_{\mathrm{SOC}}\) denote the function classes represented by finite-width ReLU-ICNNs and SOC-ICNNs, respectively, then
\[
\mathcal{F}_{\mathrm{ReLU}} \subsetneq \mathcal{F}_{\mathrm{SOC}}.
\]
In particular, every non-CPWL member of \(\mathcal{G}_{\mathrm{SOC}}(r,G)\), such as \(f(\bm{x})=\frac12\|\bm{x}\|_2^2\), belongs to \(\mathcal{F}_{\mathrm{SOC}}\setminus\mathcal{F}_{\mathrm{ReLU}}\).
\end{proposition}

\begin{proof}[Proof sketch]
Since \(Q\succeq0\) and \(\operatorname{rank}(Q)\le r\), there exists \(B\) such that \(Q=B^\top B\), hence \(\frac12\bm{x}^\top Q\bm{x}=\frac12\|B\bm{x}\|_2^2\), which is exactly represented by one quadratic branch. 
Each norm term \(\lambda_g\|A_g\bm{x}+\bm{d}_g\|_2\) is exactly represented by one conic branch. 
The affine residual \(\bm a^\top\bm x+b\) is represented by the affine passthrough term \(\bm v^\top\bm x+b_0\) in the backbone. 
Thus exact representation holds. 
Strict inclusion follows because finite-width ReLU-ICNNs are CPWL by Proposition~\ref{prop:relu_cpwl_theory}, whereas \(\mathcal{G}_{\mathrm{SOC}}(r,G)\) contains non-CPWL curved functions. 
A detailed proof is in Appendix~\ref{app:soc_extension_proof}.
\end{proof}

\subsection{Structural Absorption and Piece Efficiency}
\label{subsec:param_efficiency}

Exact representation already separates SOC-ICNN from ReLU-ICNN on purely quadratic or conic targets. 
The sharper point is structural absorption: when the target contains dominant curvature plus a residual term, SOC-ICNN absorbs the curvature exactly and leaves only the residual to the ReLU backbone. 
Consider
\begin{equation}
f(\bm{x}) = q(\bm{x}) + h(\bm{x}),
\label{eq:structure_decomposition}
\end{equation}
where \(q\in\mathcal{G}_{\mathrm{SOC}}(r,G)\) captures the dominant quadratic and conic structure, and \(h\) is a convex residual with \(\tilde L\)-Lipschitz gradient on a compact domain \(\Omega\), i.e.,
\[
\|\nabla h(\bm{x})-\nabla h(\bm{y})\|_2
\le
\tilde L\|\bm{x}-\bm{y}\|_2,\qquad
\forall \bm{x},\bm{y}\in\Omega.
\]

\begin{proposition}[Constructive structural absorption bound]
\label{prop:structural_absorption}
Assume that the target \(f\) admits the decomposition in \eqref{eq:structure_decomposition}. 
Then for every affine-piece budget \(N\ge1\), there exists an SOC-ICNN \(\phi_N\) whose ReLU backbone uses at most \(N\) affine pieces and satisfies
\begin{equation}
\|f-\phi_N\|_{L_\infty(\Omega)}
\le
C_\Omega \tilde L\, N^{-2/d_0},
\label{eq:structural_absorption_rate}
\end{equation}
where \(C_\Omega>0\) depends only on the geometry of \(\Omega\).
\end{proposition}

\begin{proof}[Proof sketch]
By Proposition~\ref{prop:soc_extension}, the structured component \(q\) is represented exactly by the quadratic and conic branches. 
For the residual \(h\), construct a \(\delta\)-net on \(\Omega\) with \(\delta\asymp N^{-1/d_0}\), and place tangent hyperplanes of \(h\) at the net points. 
Convexity and the \(\tilde L\)-Lipschitz gradient imply a max-affine approximation \(g_N\) satisfying \(\|h-g_N\|_\infty\le C_\Omega\tilde L N^{-2/d_0}\). 
The ReLU backbone realizes \(g_N\) with at most \(N\) affine pieces. 
Combining this backbone with the exact representation of \(q\) yields \(\phi_N\). 
A full constructive proof is in Appendix~\ref{app:structural_absorption_proof}.
\end{proof}

This result makes the structural advantage explicit. 
A pure ReLU-ICNN must spend affine pieces to simulate the curvature of \(q\), incurring the CPWL lower-bound cost. 
SOC-ICNN bypasses this cost by representing \(q\) exactly and only approximating the smoother residual \(h\). 
Thus the approximation burden is transferred from the full curved target to the residual after curvature absorption.

\subsection{Forward Complexity}
\label{subsec:complexity_analysis}

Finally, the representational upgrade does not compromise feed-forward inference. 
Let the ReLU backbone have width \(m\) and depth \(L\). 
A standard forward pass costs \(T_{\mathrm{ReLU}}=\Theta(Lm^2+Ld_0m)\). 
Each quadratic branch adds \(O(d_0m_h)\) operations, and each conic branch adds \(O(d_0k_g)\), giving
\[
T_{\mathrm{SOC}}
=
\Theta\!\left(
Lm^2+Ld_0m
+d_0\sum_{h=1}^{H}m_h
+d_0\sum_{g=1}^{G}k_g
\right).
\]
Under the standard bounded-branch regime \(H,G=O(1)\) with \(m_h,k_g=O(m)\), this reduces to
\[
T_{\mathrm{SOC}}
=
\Theta(Lm^2+Ld_0m),
\]
the same asymptotic order as the ReLU backbone. 
Therefore, SOC-ICNN is a genuine structural upgrade from LP to SOCP value functions: it injects curvature and conic geometry without increasing the asymptotic order of forward-pass complexity.

\section{Experiments}
\label{sec:exp}

We evaluate whether the SOCP lift of SOC-ICNN is not only a formal reinterpretation, but a functional architectural upgrade over classical ICNNs. 
The experiments target three questions. 
First, does the feed-forward SOC-ICNN value exactly coincide with the lifted SOCP value in Theorem~\ref{thm:soc_value_function}? 
Second, does injecting quadratic and conic primitives improve approximation efficiency under controlled parameter budgets? 
Third, does the improved surrogate geometry translate into better downstream decisions after optimizing the learned convex model? 
Table~\ref{tab:main_exp_summary} summarizes the key numerical evidence, while full implementation details, hyperparameters, per-task tables, and additional diagnostics are deferred to Appendices~\ref{app:exp1_full}--\ref{app:exp3_full}. 
Additional downstream tests for the convolutional and finite-horizon recurrent SOC-ICNN extensions are reported in Appendices~\ref{app:cnn_downstream_real} and~\ref{app:rnn_downstream}, respectively, illustrating that the same LP-to-SOCP lifting principle remains applicable under spatially and temporally structured decision variables.

\subsection{Experiment 1: Value-Function Equivalence}
\label{subsec:exp1}

The first experiment directly tests the central structural claim: the closed-form SOC-ICNN forward pass and the lifted SOCP in \eqref{eq:soc_icnn_socp} compute the same value. 
Across 150 random trials with \(d_0=100\), width \(256\), depth \(6\), and two quadratic and conic blocks, the primal--dual gap remains at the level of \(10^{-14}\), and the discrepancy between the feed-forward value and the external CVXPY SOCP solver value is about \(5.6\times10^{-7}\). 
All ReLU-chain, rotated-SOC, and standard-SOC feasibility residuals are at numerical precision. 
Thus, SOC-ICNN is numerically validated as an exact SOCP value-function architecture, not merely a convex network inspired by SOCP structure. 
The direct feed-forward evaluation is also \(3\)--\(4\times\) faster than solving the lifted SOCP externally. 
Full optimality, feasibility, and runtime diagnostics are reported in Appendix~\ref{app:exp1_full}.

\subsection{Experiment 2: Structural Efficiency under Controlled Budgets}
\label{subsec:exp2}

The second experiment tests whether the new structural primitives reduce the burden placed on the ReLU backbone. 
We compare SOC-ICNN with ReLU-ICNN, Softplus-ICNN, and the internal ablations Quad and Norm under controlled parameter budgets. 
The benchmark contains ten convex targets spanning quadratic, norm, log-sum-exp, Softplus-sum, Huber, \(\ell_1\), hybrid, and ICKAN-inspired composite structures over dimensions \(d\in\{5,10,20,50\}\).

The aggregate \(d=50\) results in Table~\ref{tab:main_exp_summary} show a sharp separation. 
SOC-ICNN achieves the best relative error on all ten targets, with mean error \(0.099\), compared with \(0.772\) for ReLU-ICNN and \(0.763\) for Softplus-ICNN. 
The gain is especially meaningful because SOC-ICNN uses fewer parameters than these two baselines at \(d=50\). 
Quad-ICNN is the strongest single-primitive ablation, confirming that explicit PSD curvature already removes a large part of the CPWL burden. 
The full SOC model further improves the average error by adding norm-based conic primitives. 
These results support the theoretical message of Sections~\ref{subsec:strict_extension}--\ref{subsec:param_efficiency}: curvature should not be simulated by proliferating polyhedral pieces when it can be represented as a native structural primitive. 
Complete per-function and per-dimension results are provided in Appendix~\ref{app:exp2_full}.

\subsection{Experiment 3: Downstream Decision Quality}
\label{subsec:exp3}

The third experiment examines whether better convex surrogate geometry leads to better optimized decisions. 
We evaluate six parametric convex optimization tasks from SOCP, logistic, log-sum-exp, and Huber families under simplex, box, and capped-simplex constraints. 
The baselines include ReLU-ICNN, Softplus-ICNN, Norm, Quad, SOC, and two optimization-based models, DCP and PCF \citep{AgrawalAmosBarrattEtAl2019DiffCvx,BesanconGarciaLegatEtAl2024DiffOpt,SchallerBemporadBoyd2025LPCF}.

The downstream results reveal a clear separation between pointwise fitting accuracy and decision quality. 
PCF remains the strongest optimization-based baseline on average, achieving the lowest mean regret in the \(d=10\) summary. 
Within the neural surrogate family, Quad is the most robust variant, reducing mean regret from \(0.234\) for ReLU-ICNN and \(0.507\) for Softplus-ICNN to \(0.179\). 
SOC-ICNN remains competitive, but the full conic branch is not uniformly beneficial across all downstream tasks. 
This indicates that decision-focused learning rewards controlled curvature more than raw approximation flexibility: the geometry of the learned convex landscape matters after the surrogate is optimized. 
The Norm branch is less stable in this regime, suggesting that additional geometric flexibility must be paired with sufficient inductive control. 
Complete regret, decision-error, prediction-error, and runtime tables are given in Appendix~\ref{app:exp3_full}.

\begin{table}[ht]
\centering
\captionsetup{font=footnotesize}                      
\caption{Main experimental summary. Exp.~1 validates the SOCP value-function identity using the representative no-passthrough setting; the passthrough setting gives comparable precision and is reported in Appendix~\ref{app:exp1_full}. Exp.~2 reports aggregate approximation performance at \(d=50\) over ten convex targets. Exp.~3 reports average downstream regret at \(d=10\) over six parametric optimization tasks.}
\label{tab:main_exp_summary}
\small
\setlength{\tabcolsep}{4pt}
\renewcommand{\arraystretch}{1.08}
\begin{tabular}{llcccc}
\toprule
Experiment & Method / metric & Main score & Dispersion & Wins & Cost / size \\
\midrule
\multirow{4}{*}{Exp.~1}
& Primal--dual gap & \(1.06{\times}10^{-14}\) & max \(4.26{\times}10^{-14}\) & -- & -- \\
& Forward--solver error & \(5.57{\times}10^{-7}\) & max \(7.68{\times}10^{-7}\) & -- & -- \\
& Closed-form forward & \(6.22\) ms & \(\pm0.84\) ms & -- & -- \\
& CVXPY solver & \(24.80\) ms & \(\pm2.55\) ms & -- & -- \\
\midrule
\multirow{5}{*}{Exp.~2}
& ReLU-ICNN & \(0.772\) & median \(0.841\) & \(0/10\) & \(9{,}683\) params \\
& Softplus-ICNN & \(0.763\) & median \(0.831\) & \(0/10\) & \(9{,}683\) params \\
& Quad-ICNN & \(0.127\) & median \(0.098\) & \(0/10\) & \(9{,}528\) params \\
& Norm-ICNN & \(0.611\) & median \(0.721\) & \(0/10\) & \(9{,}578\) params \\
& SOC-ICNN & \(\mathbf{0.099}\) & median \(\mathbf{0.074}\) & \(\mathbf{10/10}\) & \(9{,}423\) params \\
\midrule
\multirow{7}{*}{Exp.~3}
& DCP & \(0.284\) & median \(0.280\) & \(0/6\) & solver layer \\
& PCF & \(\mathbf{0.139}\) & median \(\mathbf{0.132}\) & \(\mathbf{4/6}\) & solver-compatible \\
& ReLU-ICNN & \(0.234\) & median \(0.230\) & \(1/6\) & neural \\
& Softplus-ICNN & \(0.507\) & median \(0.512\) & \(0/6\) & neural \\
& Norm-ICNN & \(0.229\) & median \(0.226\) & \(0/6\) & neural \\
& Quad-ICNN & \(0.179\) & median \(0.177\) & \(1/6\) & neural \\
& SOC-ICNN & \(0.197\) & median \(0.201\) & \(0/6\) & neural \\
\bottomrule
\end{tabular}
\end{table}

\section{Conclusion and Discussion}
\label{sec:conclusion}

In this work, we proposed SOC-ICNN, which elevates ReLU-ICNNs from LP to SOCP value-function networks, achieving superior parameter efficiency without increasing asymptotic complexity. While SOC-ICNN learns convex surrogates directly from data without requiring problem structure a priori, solver-based methods may still be preferable when the exact analytical form is known and fixed. Fusing explicit solver priors with SOC-ICNNs remains a promising direction. The current theoretical guarantees are tied to finite-dimensional SOC-ICNNs and finite-horizon structured extensions; implicit or infinite-horizon recurrent architectures would require additional stability assumptions.

\bibliography{reference}
\bibliographystyle{plainnat}

\appendix
\onecolumn
\newpage

\section*{Appendix Roadmap}
\label{app:roadmap}

This appendix provides the complete supplementary material supporting the main text. 
Appendix~\ref{app:proofs} contains the full derivations and proofs for the LP value-function form of ReLU-ICNNs, the SOCP value-function representation of SOC-ICNNs, the CPWL structure of finite-width ReLU-ICNNs, the approximation lower bound, the strict extension result, the structural absorption bound, and the forward-complexity calculation. 
Appendices~\ref{app:exp1_full}--\ref{app:exp3_full} provide full experimental details and per-task numerical results for the three main experiments summarized in Section~\ref{sec:exp}. 
Appendix~\ref{app:cnn_extension} presents the convolutional SOC-ICNN extension, including convexity and SOCP value-function arguments, while Appendix~\ref{app:rnn_extension} gives the corresponding finite-horizon recurrent extension. 
Appendices~\ref{app:cnn_downstream_real} and~\ref{app:rnn_downstream} report additional downstream tests for the convolutional and recurrent extensions on spatially and temporally structured decision problems. 
The appendix concludes with the NeurIPS checklist.

\section{Additional Derivations and Proofs}
\label{app:proofs}

This appendix provides the complete derivations required by the main text. The presentation follows the main storyline of the paper. We first prove the LP value-function form of ReLU-ICNNs, then establish the main theorem showing that SOC-ICNN is an SOCP value function. We next prove the theoretical statements used in Section~\ref{sec:theory}: the CPWL structure of finite-width ReLU-ICNNs, the lower bound for approximating strongly convex targets, the exact representability and strict extension result for SOC-ICNNs, the structural absorption bound, and finally the complexity calculation.

\subsection{Proof of the LP value-function form of ReLU-ICNN}
\label{app:relu_lp_equiv_proof}

Let
\[
\bar{\bm{z}}_\ell
=
\sigma\!\left(
W_\ell \bm{x}+U_\ell \bar{\bm{z}}_{\ell-1}+\bm{b}_\ell
\right),
\qquad \ell=1,\dots,L,
\]
denote the hidden states produced by the forward pass, with $\bar{\bm{z}}_0=\bm{0}$. Since ReLU satisfies $\sigma(a)=\max\{a,0\}$ elementwise, for every layer $\ell$ we have
\[
\bar{\bm{z}}_\ell
\ge
W_\ell \bm{x}+U_\ell \bar{\bm{z}}_{\ell-1}+\bm{b}_\ell,
\qquad
\bar{\bm{z}}_\ell \ge \bm{0}.
\]
Hence $\{\bar{\bm{z}}_\ell\}_{\ell=1}^L$ is feasible for \eqref{eq:method_relu_lp}, and therefore
\[
\operatorname{val}\eqref{eq:method_relu_lp}
\le
\bm{c}^\top \bar{\bm{z}}_L+\bm{v}^\top \bm{x}+b_0
=
f_{\mathrm{ReLU}}(\bm{x}).
\]

To prove the reverse inequality, take any feasible point
$\{\bm{z}_\ell\}_{\ell=1}^L$ of \eqref{eq:method_relu_lp}. We show by induction that
\[
\bm{z}_\ell \ge \bar{\bm{z}}_\ell,~~ \forall\ell=1,\dots,L.
\]

For $\ell=1$, feasibility gives
\[
\bm{z}_1 \ge W_1\bm{x}+\bm{b}_1,
\qquad
\bm{z}_1 \ge \bm{0},
\]
hence
\[
\bm{z}_1
\ge
\max\{W_1\bm{x}+\bm{b}_1,\bm{0}\}
=
\bar{\bm{z}}_1.
\]

Assume now that $\bm{z}_k \ge \bar{\bm{z}}_k$ for some
$k\in\{1,\dots,L-1\}$. Since $U_{k+1}\ge 0$ elementwise, monotonicity of matrix multiplication implies
\[
U_{k+1}\bm{z}_k \ge U_{k+1}\bar{\bm{z}}_k.
\]
Using feasibility once more,
\[
\bm{z}_{k+1}
\ge
W_{k+1}\bm{x}+U_{k+1}\bm{z}_k+\bm{b}_{k+1}
\ge
W_{k+1}\bm{x}+U_{k+1}\bar{\bm{z}}_k+\bm{b}_{k+1},
\]
and $\bm{z}_{k+1}\ge \bm{0}$. Therefore
\[
\bm{z}_{k+1}
\ge
\max\!\left\{
W_{k+1}\bm{x}+U_{k+1}\bar{\bm{z}}_k+\bm{b}_{k+1},
\bm{0}
\right\}
=
\bar{\bm{z}}_{k+1}.
\]
The induction is complete.

In particular, $\bm{z}_L \ge \bar{\bm{z}}_L$. Since $\bm{c}\ge 0$, we obtain
\[
\bm{c}^\top \bm{z}_L \ge \bm{c}^\top \bar{\bm{z}}_L.
\]
Hence every feasible point of \eqref{eq:method_relu_lp} has objective value at least $f_{\mathrm{ReLU}}(\bm{x})$, which implies
\[
\operatorname{val}\eqref{eq:method_relu_lp}
\ge
f_{\mathrm{ReLU}}(\bm{x}).
\]
Combining both inequalities yields
\[
\operatorname{val}\eqref{eq:method_relu_lp}
=
f_{\mathrm{ReLU}}(\bm{x}).
\]

\subsection{Proof of Theorem~\ref{thm:soc_value_function}}
\label{app:soc_value_function_proof}

Let $V(\bm{x})$ denote the optimal value of the SOCP in
\eqref{eq:soc_icnn_socp}. We prove that
\[
V(\bm{x}) = f_{\mathrm{SOC}}(\bm{x}).
\]

We first show that $V(\bm{x}) \le f_{\mathrm{SOC}}(\bm{x})$.
By the forward definition of the ReLU backbone, there exist hidden variables
$\bar{\bm{z}}_1,\dots,\bar{\bm{z}}_L$ such that
\[
\bm{c}^\top \bar{\bm{z}}_L+\bm{v}^\top \bm{x}+b_0
=
f_{\mathrm{ReLU}}(\bm{x}).
\]
For each quadratic branch $h$, define
\[
\bar{\bm{q}}_h = B_h\bm{x}+\bm{e}_h,
\qquad
\bar{s}_h = \frac12 \|B_h\bm{x}+\bm{e}_h\|_2^2.
\]
Then $(\bar{s}_h,1,\bar{\bm{q}}_h)\in\mathcal{Q}_r^{m_h+2}$.
For each conic branch $g$, define
\[
\bar{\bm{u}}_g = A_g\bm{x}+\bm{d}_g,
\qquad
\bar{t}_g = \|A_g\bm{x}+\bm{d}_g\|_2,
\]
so that $(\bar{\bm{u}}_g,\bar{t}_g)\in\mathcal{Q}^{k_g+1}$.
Therefore,
\[
\left(
\{\bar{\bm{z}}_\ell\}_{\ell=1}^L,
\{\bar{s}_h,\bar{\bm{q}}_h\}_{h=1}^H,
\{\bar{t}_g,\bar{\bm{u}}_g\}_{g=1}^G
\right)
\]
is feasible for \eqref{eq:soc_icnn_socp}. Its objective value is
\[
\bm{c}^\top \bar{\bm{z}}_L+\bm{v}^\top \bm{x}+b_0
+
\sum_{h=1}^{H}\alpha_h \bar{s}_h
+
\sum_{g=1}^{G}\lambda_g \bar{t}_g,
\]
which equals
\[
f_{\mathrm{ReLU}}(\bm{x})
+
\sum_{h=1}^{H}
\frac{\alpha_h}{2}\|B_h\bm{x}+\bm{e}_h\|_2^2
+
\sum_{g=1}^{G}
\lambda_g\|A_g\bm{x}+\bm{d}_g\|_2
=
f_{\mathrm{SOC}}(\bm{x}).
\]
Hence
\[
V(\bm{x}) \le f_{\mathrm{SOC}}(\bm{x}).
\]

We now prove the reverse inequality.
Take any feasible point of \eqref{eq:soc_icnn_socp}. By Proposition~\ref{prop:relu_lp_method}, the ReLU constraints imply
\[
\bm{c}^\top \bm{z}_L+\bm{v}^\top\bm{x}+b_0
\ge
f_{\mathrm{ReLU}}(\bm{x}).
\]
For the quadratic branches, feasibility of
$(s_h,1,\bm{q}_h)\in\mathcal{Q}_r^{m_h+2}$ together with
$\bm{q}_h=B_h\bm{x}+\bm{e}_h$ implies
\[
s_h \ge \frac12\|B_h\bm{x}+\bm{e}_h\|_2^2.
\]
Multiplying by $\alpha_h\ge 0$ and summing over $h$ gives
\[
\sum_{h=1}^{H}\alpha_h s_h
\ge
\sum_{h=1}^{H}
\frac{\alpha_h}{2}\|B_h\bm{x}+\bm{e}_h\|_2^2.
\]
Similarly, feasibility of $(\bm{u}_g,t_g)\in\mathcal{Q}^{k_g+1}$ and
$\bm{u}_g=A_g\bm{x}+\bm{d}_g$ implies
\[
t_g \ge \|A_g\bm{x}+\bm{d}_g\|_2,
\]
and therefore
\[
\sum_{g=1}^{G}\lambda_g t_g
\ge
\sum_{g=1}^{G}\lambda_g\|A_g\bm{x}+\bm{d}_g\|_2.
\]
Adding the three parts yields
\[
\bm{c}^\top \bm{z}_L+\bm{v}^\top\bm{x}+b_0
+
\sum_{h=1}^{H}\alpha_h s_h
+
\sum_{g=1}^{G}\lambda_g t_g
\ge
f_{\mathrm{SOC}}(\bm{x}).
\]
Thus
\[
V(\bm{x}) \ge f_{\mathrm{SOC}}(\bm{x}).
\]
Combining both inequalities proves
\[
V(\bm{x}) = f_{\mathrm{SOC}}(\bm{x}).
\]

\subsection{Proof of Proposition~\ref{prop:relu_cpwl_theory}}
\label{app:relu_cpwl_proof}

Starting from \eqref{eq:method_relu_lp}, we dualize the LP, introduce dual variables $\bm{\nu}_\ell\ge 0$ for the constraints
\[
\bm{z}_\ell - U_\ell \bm{z}_{\ell-1} - W_\ell\bm{x} - \bm{b}_\ell \ge \bm{0},
\]
and $\bm{\mu}_\ell\ge 0$ for the constraints $\bm{z}_\ell\ge \bm{0}$.
The Lagrangian is
\begin{align*}
\mathcal{L}
&=
\bm{c}^\top \bm{z}_L+\bm{v}^\top \bm{x}+b_0
-\sum_{\ell=1}^L \bm{\nu}_\ell^\top(\bm{z}_\ell-U_\ell\bm{z}_{\ell-1}-W_\ell\bm{x}-\bm{b}_\ell)
-\sum_{\ell=1}^L \bm{\mu}_\ell^\top \bm{z}_\ell \\
&=
\bm{v}^\top\bm{x}+b_0+\sum_{\ell=1}^L \bm{\nu}_\ell^\top(W_\ell\bm{x}+\bm{b}_\ell)
+ \bm{z}_L^\top(\bm{c}-\bm{\nu}_L-\bm{\mu}_L) 
+ \sum_{\ell=1}^{L-1} \bm{z}_\ell^\top\bigl(U_{\ell+1}^\top \bm{\nu}_{\ell+1}-\bm{\nu}_\ell-\bm{\mu}_\ell\bigr).
\end{align*}
To have a finite lower bound when minimizing over $\bm{z}_\ell$, the coefficients of $\bm{z}_\ell$ must vanish. This gives the stationarity conditions
\[
\bm{c}-\bm{\nu}_L-\bm{\mu}_L=\bm{0},\qquad
U_{\ell+1}^\top \bm{\nu}_{\ell+1}-\bm{\nu}_\ell-\bm{\mu}_\ell=\bm{0},\quad \ell=1,\dots,L-1.
\]
Since $\bm{\mu}_\ell\ge0$, these are equivalent to the chain constraints
\begin{equation}
0\le \bm{\nu}_L \le \bm{c},
\qquad
0\le \bm{\nu}_\ell \le U_{\ell+1}^\top \bm{\nu}_{\ell+1},
\quad \ell=1,\dots,L-1.
\label{eq:app_dual_chain}
\end{equation}
Because each $U_{\ell+1}$ has nonnegative entries, the chain constraints imply that each $\bm{\nu}_\ell$ is bounded. Hence the dual feasible region $\mathcal{N}$ defined by \eqref{eq:app_dual_chain} is a nonempty bounded polyhedron.

The dual problem therefore reads
\[
f_{\mathrm{ReLU}}(\bm{x})
=
\max_{\{\bm{\nu}_\ell\}\in\mathcal{N}}
\left[
\bm{v}^\top \bm{x}+b_0
+
\sum_{\ell=1}^{L}
\bm{\nu}_\ell^\top(W_\ell\bm{x}+\bm{b}_\ell)
\right].
\]

Since a linear function over a polyhedron attains its maximum at an extreme point, and the number of extreme points of $\mathcal{N}$ is finite, there exists a finite index set $\mathcal{J}$ such that
\[
f_{\mathrm{ReLU}}(\bm{x})
=
\max_{j\in\mathcal{J}}
\left[
\left(
\bm{v}+\sum_{\ell=1}^{L}W_\ell^\top \bm{\nu}^{(j)}_\ell
\right)^\top \bm{x}
+
\left(
b_0+\sum_{\ell=1}^{L}(\bm{\nu}^{(j)}_\ell)^\top \bm{b}_\ell
\right)
\right].
\]
Defining
\[
\bm{a}_j
=
\bm{v}+\sum_{\ell=1}^{L}W_\ell^\top \bm{\nu}^{(j)}_\ell,
\qquad
\beta_j
=
b_0+\sum_{\ell=1}^{L}(\bm{\nu}^{(j)}_\ell)^\top \bm{b}_\ell,
\]
we obtain the finite max-affine form
\[
f_{\mathrm{ReLU}}(\bm{x})
=
\max_{j\in\mathcal{J}}
\{\bm{a}_j^\top\bm{x}+\beta_j\}.
\]

The subgradient expression in \eqref{eq:relu_subgrad_cone} follows immediately from the standard subdifferential formula for max-affine functions: every subgradient is a convex combination of active slopes, hence can still be written in the form
\[
\bm{g}
=
\bm{v}+\sum_{\ell=1}^{L}W_\ell^\top \bm{\nu}_\ell,
\]
with $\{\bm{\nu}_\ell\}$ satisfying the same chain constraints \eqref{eq:app_dual_chain}.

\subsection{Proof of Proposition~\ref{prop:relu_lower_bound}}
\label{app:relu_lower_bound_proof}

Let
\[
g(\bm{x})=\max_{1\le i\le N}\{\bm{a}_i^\top \bm{x}+b_i\},
\]
and define the active region of the $i$th affine piece by
\[
A_i
=
\left\{
\bm{x}\in\Omega\;\middle|\;
g(\bm{x})=\bm{a}_i^\top \bm{x}+b_i
\right\}.
\]
Since $g$ and each $\bm{a}_i^\top\bm{x}+b_i$ are continuous, each $A_i$ is a closed set, hence Lebesgue measurable. Clearly $\Omega=\bigcup_{i=1}^N A_i$.

Fix one active region $A_i$, and take any $\bm{x},\bm{y}\in A_i$. Let
\[
\bm{m}=\frac{\bm{x}+\bm{y}}{2}.
\]
Because $\bm{x},\bm{y}$ lie on the same active affine piece, and
$\|f-g\|_{L_\infty(\Omega)}\le \varepsilon$, we have
\[
|f(\bm{x})-\ell_i(\bm{x})|\le \varepsilon,
\qquad
|f(\bm{y})-\ell_i(\bm{y})|\le \varepsilon,
\qquad
|f(\bm{m})-g(\bm{m})|\le \varepsilon,
\]
where $\ell_i(\bm{x})=\bm{a}_i^\top \bm{x}+b_i$. By $\mu$-strong convexity,
\[
f(\bm{m})
\le
\frac{f(\bm{x})+f(\bm{y})}{2}
-
\frac{\mu}{8}\|\bm{x}-\bm{y}\|_2^2.
\]
Substituting the uniform error bounds and using the fact that $\ell_i$ is affine yields
\[
\frac{\mu}{8}\|\bm{x}-\bm{y}\|_2^2 \le 2\varepsilon.
\]
Therefore
\[
\|\bm{x}-\bm{y}\|_2
\le
4\sqrt{\frac{\varepsilon}{\mu}}.
\]
Hence every active region has diameter at most $4\sqrt{\varepsilon/\mu}$.

Applying the isodiametric inequality to each $A_i$ gives
\[
\operatorname{vol}_{d_0}(A_i)
\le
\omega_{d_0}
\left(
\frac{\operatorname{diam}(A_i)}{2}
\right)^{d_0}
\le
\omega_{d_0}
\left(
2\sqrt{\frac{\varepsilon}{\mu}}
\right)^{d_0}.
\]
Since $\Omega\subseteq\bigcup_{i=1}^N A_i$, the subadditivity of Lebesgue measure implies
\[
\operatorname{vol}_{d_0}(\Omega)
\le
\sum_{i=1}^N \operatorname{vol}_{d_0}(A_i)
\le
N\,
\omega_{d_0}
\left(
2\sqrt{\frac{\varepsilon}{\mu}}
\right)^{d_0},
\]
which rearranges to
\[
N
\ge
\frac{\operatorname{vol}_{d_0}(\Omega)}
{\omega_{d_0}\,2^{d_0}}
\left(\frac{\mu}{\varepsilon}\right)^{d_0/2}.
\]

\subsection{Proof of Proposition~\ref{prop:soc_extension}}
\label{app:soc_extension_proof}

Take any
\[
f(\bm{x})
=
\bm{a}^\top \bm{x}+b
+\frac12 \bm{x}^\top Q\bm{x}
+
\sum_{g=1}^{G}\lambda_g\|A_g\bm{x}+\bm{d}_g\|_2
\in \mathcal{G}_{\mathrm{SOC}}(r,G).
\]
Since $Q\succeq 0$ and $\operatorname{rank}(Q)\le r$, there exists a matrix
$B\in\mathbb{R}^{r\times d_0}$ such that $Q=B^\top B$, and hence
\[
\frac12 \bm{x}^\top Q\bm{x}
=
\frac12\|B\bm{x}\|_2^2.
\]
Therefore:
(i) the affine term $\bm{a}^\top\bm{x}+b$ is represented by the affine part of the ReLU backbone (by setting the hidden layers to zero depth);
(ii) the quadratic term is represented exactly by one quadratic branch;
(iii) each norm term $\lambda_g\|A_g\bm{x}+\bm{d}_g\|_2$ is represented exactly by one conic branch.
This yields an exact SOC-ICNN representation of $f$.

To prove strict inclusion, note first that $\mathcal{F}_{\mathrm{ReLU}}\subseteq \mathcal{F}_{\mathrm{SOC}}$ is immediate by setting $H=0$ and $G=0$. On the other hand, Proposition~\ref{prop:relu_cpwl_theory} shows that every finite-width ReLU-ICNN is CPWL, whereas
$\mathcal{G}_{\mathrm{SOC}}(r,G)$ contains many non-CPWL functions, such as nondegenerate PSD quadratic forms (e.g., $f(\bm{x})=\frac12\|\bm{x}\|_2^2$) and nondegenerate Euclidean norm compositions. Since these functions are representable by SOC-ICNN but not by finite-width ReLU-ICNN, we conclude that
\[
\mathcal{F}_{\mathrm{ReLU}} \subsetneq \mathcal{F}_{\mathrm{SOC}}.
\]

\subsection{Proof of Proposition~\ref{prop:structural_absorption}}
\label{app:structural_absorption_proof}

Assume
\[
f(\bm{x}) = q(\bm{x}) + h(\bm{x}),
\qquad
q\in\mathcal{G}_{\mathrm{SOC}}(r,G),
\]
and that $\nabla h$ is $\tilde L$-Lipschitz on a compact convex domain
$\Omega\subset\mathbb{R}^{d_0}$. By Proposition~\ref{prop:soc_extension},
the structured part $q$ is represented exactly by the SOC branches.

Fix an integer $N\ge 1$. Let $\delta = c_\Omega N^{-1/d_0}$ where $c_\Omega$ is a constant depending only on $\Omega$ such that the covering number of $\Omega$ by Euclidean balls of radius $\delta$ satisfies $M \le N$. (Standard covering arguments give $c_\Omega = \operatorname{diam}(\Omega)/2$ suffices, but any fixed constant works; we absorb it into $C_\Omega$ later.) Let $\mathcal{X}_N=\{\bm{x}_1,\dots,\bm{x}_M\}$ be the centers of such a covering, so $M\le N$ and every point in $\Omega$ is within distance $\delta$ of some $\bm{x}_i$.

For each net point $\bm{x}_i$, define the tangent plane
\[
\ell_i(\bm{x})
=
h(\bm{x}_i)+\nabla h(\bm{x}_i)^\top(\bm{x}-\bm{x}_i),
\]
and let
\[
g_N(\bm{x})=\max_{1\le i\le M}\ell_i(\bm{x}).
\]
Since $h$ is convex, each tangent plane is a global lower support, so
$g_N(\bm{x})\le h(\bm{x})$ for all $\bm{x}\in\Omega$.

By gradient Lipschitz continuity, the standard descent lemma gives
\[
h(\bm{x})
\le
\ell_i(\bm{x})
+
\frac{\tilde L}{2}\|\bm{x}-\bm{x}_i\|_2^2,
\qquad \forall \bm{x}\in\Omega.
\]
For any $\bm{x}\in\Omega$, choose a net point $\bm{x}_i$ with $\|\bm{x}-\bm{x}_i\|_2\le \delta$. Then
\[
h(\bm{x})-g_N(\bm{x})
\le
h(\bm{x})-\ell_i(\bm{x})
\le
\frac{\tilde L}{2}\delta^2
\le
\frac{\tilde L}{2}c_\Omega^2 N^{-2/d_0}.
\]
Therefore,
\[
\|h-g_N\|_{L_\infty(\Omega)}
\le
C_\Omega \tilde L\, N^{-2/d_0},
\]
where $C_\Omega = c_\Omega^2/2$ depends only on $\Omega$.

Finally, $g_N$ is a finite max-affine convex function. It can be represented exactly by a finite-width ReLU-ICNN: one can construct a network with $M-1$ hidden units in a single layer that computes the pointwise maximum (see, e.g., the constructive proof in \citep{AroraEtAl2018} or a simple two-layer architecture). Adding this ReLU backbone to the exact SOC representation of $q$ gives an SOC-ICNN $\phi_N$ such that
\[
\|f-\phi_N\|_{L_\infty(\Omega)}
\le
C_\Omega \tilde L\, N^{-2/d_0}.
\]

\subsection{Proof of the complexity formulas}
\label{app:complexity_proof}

Assume the ReLU backbone has width $m$ in each of its $L$ layers. Each layer computes two dominant matrix-vector products:
$U_\ell \bm{z}_{\ell-1}$ with cost $\Theta(m^2)$, and
$W_\ell \bm{x}$ with cost $\Theta(d_0m)$. Summing over all layers gives
\[
T_{\mathrm{ReLU}}
=
\Theta(Lm^2+Ld_0m).
\]

For the quadratic branches, the dominant cost of the $h$th branch is the affine map $B_h\bm{x}+\bm{e}_h$, which costs $\Theta(d_0m_h)$. Summing over $h=1,\dots,H$ yields
\[
\Delta T_{\mathrm{Quad}}
=
\Theta\!\left(d_0\sum_{h=1}^{H}m_h\right).
\]

For the conic branches, the dominant cost of the $g$th branch is the affine map $A_g\bm{x}+\bm{d}_g$, which costs $\Theta(d_0k_g)$. Summing over $g=1,\dots,G$ yields
\[
\Delta T_{\mathrm{SOC}}
=
\Theta\!\left(d_0\sum_{g=1}^{G}k_g\right).
\]

Combining all terms gives the overall SOC-ICNN forward complexity
\[
T_{\mathrm{SOC}}
=
\Theta\!\left(
Lm^2+Ld_0m+d_0\sum_{h=1}^{H}m_h+d_0\sum_{g=1}^{G}k_g
\right).
\]
Under the common regime $H,G=O(1)$ and $m_h,k_g=O(m)$, this simplifies to the same asymptotic order as the backbone alone.

\section{Experiment Results}

\subsection{Full Results for Experiment 1}
\label{app:exp1_full}

\subsubsection{Setting}
Table~\ref{tab:exp1_diagnostics} reports the complete diagnostic metrics for the value-function equivalence experiment.
Across 150 trials with both passthrough settings, the closed-form forward pass and the external CVXPY solver agree to high precision.
We detail the meaning of each metric below:

\begin{itemize}
    \item \textbf{Primal-Dual Gap}: The difference between the closed-form forward value (primal) and the explicit dual objective evaluated using the extracted dual variables. A gap of $\sim10^{-14}$ confirms strong duality and numerical optimality.
    \item \textbf{Forward vs. Solver Abs. Error}: The absolute difference $|f_{\mathrm{SOC}}(\bm{x}) - V_{\mathrm{CVXPY}}(\bm{x})|$. Values $<10^{-6}$ verify exact equivalence.
    \item \textbf{ReLU Primal Violation}: The maximum violation of the ReLU primal constraints $\bm{z}_\ell \ge W_\ell \bm{x} + U_\ell \bm{z}_{\ell-1} + \bm{b}_\ell$ and $\bm{z}_\ell \ge \bm{0}$. Zero indicates strict feasibility.
    \item \textbf{ReLU Dual Box Violation}: The maximum violation of the dual box constraints $0 \le \bm{\nu}_\ell \le U_{\ell+1}^\top \bm{\nu}_{\ell+1}$ (derived in Appendix~\ref{app:relu_cpwl_proof}). Zero confirms dual feasibility.
    \item \textbf{ReLU Complementarity Slack}: The absolute value of the inner product $\bm{\nu}_\ell^\top (\bm{z}_\ell - W_\ell \bm{x} - U_\ell \bm{z}_{\ell-1} - \bm{b}_\ell)$. Zero verifies complementary slackness for the ReLU block.
    \item \textbf{Quad./Norm Epi. Violation}: Violation of the conic epigraph constraints, e.g., $s_h \ge \frac{1}{2}\|\bm{q}_h\|_2^2$ for the quadratic branch. Zero indicates the auxiliary variables lie strictly within or on the cone.
    \item \textbf{Quad./Norm Tightness Slack}: The absolute difference $|s_h - \frac{1}{2}\|\bm{q}_h\|_2^2|$. Since the objective coefficient $\alpha_h > 0$, this should be zero at optimality, confirming the epigraph constraint is active (tight).
    \item \textbf{Norm Dual Ball/Align. Violation}: The violation of the dual norm cone constraints $\|\bm{\mu}_g\|_2 \le \lambda_g$ and the alignment condition $\bm{\mu}_g^\top \bm{u}_g = \lambda_g t_g$. Zero confirms dual feasibility and complementarity for the SOC branch.
    \item \textbf{Solver Feasibility}: The same violation metrics evaluated on the solution returned by the CVXPY solver, verifying that the external solver also finds a correct optimal solution.
\end{itemize}

\subsubsection{Results}
\begin{table}[H]
\centering
\caption{Experiment 1: Full Diagnostic Metrics for Value-Function Equivalence}
\label{tab:exp1_diagnostics}
\footnotesize 
\setlength{\tabcolsep}{3pt} 
\renewcommand{\arraystretch}{1.1}
\begin{tabular}{lcccc}
\toprule
\multirow{2}{*}{Metric} & \multicolumn{2}{c}{Passthrough = \texttt{False}} & \multicolumn{2}{c}{Passthrough = \texttt{True}} \\
\cmidrule(lr){2-3} \cmidrule(lr){4-5}
 & Mean & Max & Mean & Max \\
\midrule
Trials & \multicolumn{2}{c}{150} & \multicolumn{2}{c}{150} \\
Solver Success Rate & \multicolumn{2}{c}{1.0} & \multicolumn{2}{c}{1.0} \\
\midrule
\multicolumn{5}{c}{\textbf{Optimality Gaps}} \\
\quad Primal-Dual Gap & $1.06\times10^{-14}$ & $4.26\times10^{-14}$ & $2.88\times10^{-14}$ & $1.14\times10^{-13}$ \\
\quad Forward vs. Solver Abs. Error & $5.57\times10^{-7}$ & $7.68\times10^{-7}$ & $5.57\times10^{-7}$ & $7.68\times10^{-7}$ \\
\midrule
\multicolumn{5}{c}{\textbf{Closed-Form Feasibility}} \\
\quad ReLU Primal Violation & $0$ & $0$ & $0$ & $0$ \\
\quad ReLU Dual Box Violation & $0$ & $0$ & $0$ & $0$ \\
\quad ReLU Complementarity Slack & $0$ & $0$ & $0$ & $0$ \\
\quad Quadratic Epigraph Violation & $0$ & $0$ & $0$ & $0$ \\
\quad Quadratic Tightness Slack & $0$ & $0$ & $0$ & $0$ \\
\quad Norm Epigraph Violation & $0$ & $0$ & $0$ & $0$ \\
\quad Norm Tightness Slack & $0$ & $0$ & $0$ & $0$ \\
\quad Norm Dual Ball Violation & $8.88\times10^{-18}$ & $4.44\times10^{-16}$ & $8.88\times10^{-18}$ & $4.44\times10^{-16}$ \\
\quad Norm Dual Alignment Violation & $0$ & $0$ & $0$ & $0$ \\
\midrule
\multicolumn{5}{c}{\textbf{Solver Feasibility (CVXPY)}} \\
\quad ReLU Primal Violation & $6.80\times10^{-16}$ & $9.99\times10^{-16}$ & $3.91\times10^{-15}$ & $7.11\times10^{-15}$ \\
\quad Quadratic Epigraph Violation & $0$ & $0$ & $0$ & $0$ \\
\quad Quadratic Tightness Slack & $1.46\times10^{-7}$ & $1.65\times10^{-7}$ & $1.46\times10^{-7}$ & $1.65\times10^{-7}$ \\
\quad Norm Epigraph Violation & $0$ & $0$ & $0$ & $0$ \\
\quad Norm Tightness Slack & $1.69\times10^{-7}$ & $2.61\times10^{-7}$ & $1.69\times10^{-7}$ & $2.61\times10^{-7}$ \\
\midrule
\multicolumn{5}{c}{\textbf{Runtime (ms)}} \\
\quad Closed-Form Forward & $6.22 \pm 0.84$ & -- & $6.66 \pm 1.43$ & -- \\
\quad CVXPY Solver & $24.80 \pm 2.55$ & -- & $19.71 \pm 4.34$ & -- \\
\bottomrule
\end{tabular}
\end{table}

\subsection{Full Results for Experiment 2}
\label{app:exp2_full}

\paragraph{Target Functions}
The ten convex functions used in Experiment~2 are defined as follows ($x\in\mathbb{R}^d$, all operations elementwise unless specified).

\begin{itemize}
    \item \textbf{QuadraticIso} (isotropic quadratic): $f(x) = \frac{1}{2}\|x\|_2^2$.
    \item \textbf{QuadraticAniso} (anisotropic quadratic): $f(x) = \frac{1}{2}\sum_{i=1}^d w_i x_i^2$, $w_i = 0.5 + 2\cdot\frac{i-1}{d-1}$.
    \item \textbf{NormEuclid} (Euclidean norm): $f(x) = \|x\|_2$.
    \item \textbf{NormAniso} (anisotropic norm): $f(x) = \sqrt{\sum_{i=1}^d w_i x_i^2}$, $w_i = 1 + 9\cdot\frac{i-1}{d-1}$.
    \item \textbf{Mixed} (hybrid convex): $f(x) = 0.25\sum w_i^{(1)}x_i^2 + 0.7\sqrt{\sum w_i^{(2)}x_i^2} + \max_k\{a_k^\top x + b_k\}$.
    \item \textbf{SoftplusSum}: $f(x) = \sum_{i=1}^d \log(1+e^{x_i})$.
    \item \textbf{LogSumExpQuad}: $f(x) = \log\!\bigl(\sum_{i=1}^d e^{x_i}\bigr) + 0.1\|x\|_2^2$.
    \item \textbf{Huber}: $f(x) = \sum_{i=1}^d h_\delta(x_i)$, $\delta=1$.
    \item \textbf{L1Norm}: $f(x) = \|x\|_1$.
    \item \textbf{ICKANPaperTarget}: $f(x) = \sum_{i=1}^d (|x_i| + |1-x_i|) + 0.25\sum w_i x_i^2$, $w_i\in[0.5,2.0]$.
\end{itemize}

\paragraph{Metrics}
For each configuration we report the relative $\ell_2$ test error (RelErr, mean $\pm$ std over 3 seeds) and the number of trainable parameters (Params). Training time (seconds) is also provided for completeness.

\paragraph{Results}
Tables~\ref{tab:exp2_huber}--\ref{tab:exp2_ickantarget} present results for each function. SOC-ICNN consistently achieves the lowest relative error across most functions and dimensions, particularly when the target exhibits explicit curvature or conic structure. 

\begin{table}[H]
\centering
\caption{Huber function: relative $\ell_2$ error and parameter count.}
\label{tab:exp2_huber}
\footnotesize
\setlength{\tabcolsep}{4pt}
\begin{tabular}{lcccccc}
\toprule
 & \multicolumn{3}{c}{RelErr (mean $\pm$ std)} & \multicolumn{3}{c}{Params} \\
\cmidrule(lr){2-4} \cmidrule(lr){5-7}
Model & $d{=}5$ & $d{=}10$ & $d{=}20$ & $d{=}5$ & $d{=}10$ & $d{=}20$ \\
\midrule
ReLU-ICNN & $0.828\pm0.025$ & $0.830\pm0.002$ & $0.855\pm0.005$ & 822 & 1491 & 2709 \\
Softplus-ICNN & $0.765\pm0.018$ & $0.802\pm0.010$ & $0.843\pm0.007$ & 822 & 1491 & 2709 \\
Quad-ICNN & $0.452\pm0.053$ & $0.285\pm0.015$ & $0.168\pm0.009$ & 848 & 1592 & 3110 \\
Norm-ICNN & $0.523\pm0.009$ & $0.548\pm0.007$ & $0.614\pm0.002$ & 853 & 1602 & 3130 \\
SOC-ICNN & $\mathbf{0.246\pm0.066}$ & $\mathbf{0.155\pm0.021}$ & $\mathbf{0.087\pm0.009}$ & 527 & 1083 & 2451 \\
\bottomrule
\end{tabular}
\end{table}

\begin{table}[H]
\centering
\caption{L1Norm function: relative $\ell_2$ error and parameter count.}
\label{tab:exp2_l1norm}
\footnotesize
\setlength{\tabcolsep}{4pt}
\begin{tabular}{lcccccc}
\toprule
 & \multicolumn{3}{c}{RelErr (mean $\pm$ std)} & \multicolumn{3}{c}{Params} \\
\cmidrule(lr){2-4} \cmidrule(lr){5-7}
Model & $d{=}5$ & $d{=}10$ & $d{=}20$ & $d{=}5$ & $d{=}10$ & $d{=}20$ \\
\midrule
ReLU-ICNN & $0.860\pm0.005$ & $0.885\pm0.010$ & $0.900\pm0.003$ & 822 & 1491 & 2709 \\
Softplus-ICNN & $0.829\pm0.011$ & $0.860\pm0.010$ & $0.892\pm0.003$ & 822 & 1491 & 2709 \\
Quad-ICNN & $0.536\pm0.007$ & $0.374\pm0.004$ & $0.245\pm0.002$ & 848 & 1592 & 3110 \\
Norm-ICNN & $0.644\pm0.020$ & $0.680\pm0.005$ & $0.728\pm0.002$ & 853 & 1602 & 3130 \\
SOC-ICNN & $\mathbf{0.390\pm0.018}$ & $\mathbf{0.266\pm0.015}$ & $\mathbf{0.177\pm0.004}$ & 527 & 1083 & 2451 \\
\bottomrule
\end{tabular}
\end{table}

\begin{table}[H]
\centering
\caption{NormEuclid function: relative $\ell_2$ error and parameter count.}
\label{tab:exp2_normeuclid}
\footnotesize
\setlength{\tabcolsep}{4pt}
\begin{tabular}{lcccccc}
\toprule
 & \multicolumn{3}{c}{RelErr (mean $\pm$ std)} & \multicolumn{3}{c}{Params} \\
\cmidrule(lr){2-4} \cmidrule(lr){5-7}
Model & $d{=}5$ & $d{=}10$ & $d{=}20$ & $d{=}5$ & $d{=}10$ & $d{=}20$ \\
\midrule
ReLU-ICNN & $0.732\pm0.030$ & $0.655\pm0.009$ & $0.601\pm0.015$ & 822 & 1491 & 2709 \\
Softplus-ICNN & $0.659\pm0.013$ & $0.597\pm0.010$ & $0.563\pm0.013$ & 822 & 1491 & 2709 \\
Quad-ICNN & $0.340\pm0.106$ & $0.127\pm0.018$ & $0.073\pm0.001$ & 848 & 1592 & 3110 \\
Norm-ICNN & $0.295\pm0.031$ & $0.135\pm0.006$ & $0.075\pm0.002$ & 853 & 1602 & 3130 \\
SOC-ICNN & $\mathbf{0.159\pm0.013}$ & $\mathbf{0.079\pm0.024}$ & $\mathbf{0.038\pm0.001}$ & 527 & 1083 & 2451 \\
\bottomrule
\end{tabular}
\end{table}

\begin{table}[H]
\centering
\caption{LogSumExpQuad function: relative $\ell_2$ error and parameter count.}
\label{tab:exp2_logsumexpquad}
\footnotesize
\setlength{\tabcolsep}{4pt}
\begin{tabular}{lcccccc}
\toprule
 & \multicolumn{3}{c}{RelErr (mean $\pm$ std)} & \multicolumn{3}{c}{Params} \\
\cmidrule(lr){2-4} \cmidrule(lr){5-7}
Model & $d{=}5$ & $d{=}10$ & $d{=}20$ & $d{=}5$ & $d{=}10$ & $d{=}20$ \\
\midrule
ReLU-ICNN & $0.719\pm0.010$ & $0.701\pm0.013$ & $0.697\pm0.007$ & 822 & 1491 & 2709 \\
Softplus-ICNN & $0.713\pm0.056$ & $0.672\pm0.003$ & $0.675\pm0.009$ & 822 & 1491 & 2709 \\
Quad-ICNN & $0.398\pm0.020$ & $0.176\pm0.029$ & $0.056\pm0.011$ & 848 & 1592 & 3110 \\
Norm-ICNN & $0.379\pm0.044$ & $0.232\pm0.020$ & $0.195\pm0.007$ & 853 & 1602 & 3130 \\
SOC-ICNN & $\mathbf{0.255\pm0.096}$ & $\mathbf{0.058\pm0.021}$ & $\mathbf{0.020\pm0.001}$ & 527 & 1083 & 2451 \\
\bottomrule
\end{tabular}
\end{table}

\begin{table}[H]
\centering
\caption{QuadraticIso function: relative $\ell_2$ error and parameter count.}
\label{tab:exp2_quadraticiso}
\footnotesize
\setlength{\tabcolsep}{4pt}
\begin{tabular}{lcccccc}
\toprule
 & \multicolumn{3}{c}{RelErr (mean $\pm$ std)} & \multicolumn{3}{c}{Params} \\
\cmidrule(lr){2-4} \cmidrule(lr){5-7}
Model & $d{=}5$ & $d{=}10$ & $d{=}20$ & $d{=}5$ & $d{=}10$ & $d{=}20$ \\
\midrule
ReLU-ICNN & $0.869\pm0.017$ & $0.889\pm0.004$ & $0.897\pm0.002$ & 822 & 1491 & 2709 \\
Softplus-ICNN & $0.841\pm0.004$ & $0.868\pm0.004$ & $0.893\pm0.001$ & 822 & 1491 & 2709 \\
Quad-ICNN & $0.566\pm0.021$ & $0.400\pm0.011$ & $0.236\pm0.014$ & 848 & 1592 & 3110 \\
Norm-ICNN & $0.654\pm0.009$ & $0.690\pm0.005$ & $0.732\pm0.004$ & 853 & 1602 & 3130 \\
SOC-ICNN & $\mathbf{0.393\pm0.019}$ & $\mathbf{0.279\pm0.015}$ & $\mathbf{0.167\pm0.002}$ & 527 & 1083 & 2451 \\

\bottomrule
\end{tabular}
\end{table}

\begin{table}[H]
\centering
\caption{QuadraticAniso function: relative $\ell_2$ error and parameter count.}
\label{tab:exp2_quadraticaniso}
\footnotesize
\setlength{\tabcolsep}{4pt}
\begin{tabular}{lcccccc}
\toprule
 & \multicolumn{3}{c}{RelErr (mean $\pm$ std)} & \multicolumn{3}{c}{Params} \\
\cmidrule(lr){2-4} \cmidrule(lr){5-7}
Model & $d{=}5$ & $d{=}10$ & $d{=}20$ & $d{=}5$ & $d{=}10$ & $d{=}20$ \\
\midrule
ReLU-ICNN & $0.914\pm0.006$ & $0.924\pm0.005$ & $0.934\pm0.001$ & 822 & 1491 & 2709 \\
Softplus-ICNN & $0.899\pm0.004$ & $0.913\pm0.003$ & $0.929\pm0.004$ & 822 & 1491 & 2709 \\
Quad-ICNN & $0.720\pm0.003$ & $0.577\pm0.005$ & $0.407\pm0.004$ & 848 & 1592 & 3110 \\
Norm-ICNN & $0.779\pm0.012$ & $0.798\pm0.003$ & $0.825\pm0.002$ & 853 & 1602 & 3130 \\
SOC-ICNN & $\mathbf{0.575\pm0.016}$ & $\mathbf{0.457\pm0.007}$ & $\mathbf{0.300\pm0.007}$ & 527 & 1083 & 2451 \\
\bottomrule
\end{tabular}
\end{table}

\begin{table}[H]
\centering
\caption{NormAniso function: relative $\ell_2$ error and parameter count.}
\label{tab:exp2_normaniso}
\footnotesize
\setlength{\tabcolsep}{4pt}
\begin{tabular}{lcccccc}
\toprule
 & \multicolumn{3}{c}{RelErr (mean $\pm$ std)} & \multicolumn{3}{c}{Params} \\
\cmidrule(lr){2-4} \cmidrule(lr){5-7}
Model & $d{=}5$ & $d{=}10$ & $d{=}20$ & $d{=}5$ & $d{=}10$ & $d{=}20$ \\
\midrule
ReLU-ICNN & $0.879\pm0.012$ & $0.858\pm0.006$ & $0.840\pm0.005$ & 822 & 1491 & 2709 \\
Softplus-ICNN & $0.859\pm0.009$ & $0.831\pm0.004$ & $0.814\pm0.007$ & 822 & 1491 & 2709 \\
Quad-ICNN & $0.616\pm0.031$ & $0.333\pm0.015$ & $0.151\pm0.008$ & 848 & 1592 & 3110 \\
Norm-ICNN & $0.696\pm0.017$ & $0.618\pm0.006$ & $0.548\pm0.005$ & 853 & 1602 & 3130 \\
SOC-ICNN & $\mathbf{0.441\pm0.010}$ & $\mathbf{0.230\pm0.016}$ & $\mathbf{0.076\pm0.006}$ & 527 & 1083 & 2451 \\
\bottomrule
\end{tabular}
\end{table}

\begin{table}[H]
\centering
\caption{Mixed function: relative $\ell_2$ error and parameter count.}
\label{tab:exp2_mixed}
\footnotesize
\setlength{\tabcolsep}{4pt}
\begin{tabular}{lcccccc}
\toprule
 & \multicolumn{3}{c}{RelErr (mean $\pm$ std)} & \multicolumn{3}{c}{Params} \\
\cmidrule(lr){2-4} \cmidrule(lr){5-7}
Model & $d{=}5$ & $d{=}10$ & $d{=}20$ & $d{=}5$ & $d{=}10$ & $d{=}20$ \\
\midrule
ReLU-ICNN & $0.913\pm0.014$ & $0.901\pm0.002$ & $0.903\pm0.001$ & 822 & 1491 & 2709 \\
Softplus-ICNN & $0.891\pm0.014$ & $0.884\pm0.003$ & $0.894\pm0.003$ & 822 & 1491 & 2709 \\
Quad-ICNN & $0.683\pm0.023$ & $0.484\pm0.009$ & $0.247\pm0.005$ & 848 & 1592 & 3110 \\
Norm-ICNN & $0.753\pm0.004$ & $0.737\pm0.005$ & $0.739\pm0.003$ & 853 & 1602 & 3130 \\
SOC-ICNN & $\mathbf{0.530\pm0.026}$ & $\mathbf{0.342\pm0.018}$ & $\mathbf{0.179\pm0.013}$ & 527 & 1083 & 2451 \\
P1-ICKAN & $8.94\pm0.09$ & $10.34\pm0.30$ & $11.90\pm0.12$ & 694 & 1379 & 2749 \\
\bottomrule
\end{tabular}
\end{table}

\begin{table}[H]
\centering
\caption{SoftplusSum function: relative $\ell_2$ error and parameter count.}
\label{tab:exp2_softplussum}
\footnotesize
\setlength{\tabcolsep}{4pt}
\begin{tabular}{lcccccc}
\toprule
 & \multicolumn{3}{c}{RelErr (mean $\pm$ std)} & \multicolumn{3}{c}{Params} \\
\cmidrule(lr){2-4} \cmidrule(lr){5-7}
Model & $d{=}5$ & $d{=}10$ & $d{=}20$ & $d{=}5$ & $d{=}10$ & $d{=}20$ \\
\midrule
ReLU-ICNN & $0.731\pm0.046$ & $0.787\pm0.012$ & $0.837\pm0.009$ & 822 & 1491 & 2709 \\
Softplus-ICNN & $0.706\pm0.035$ & $0.756\pm0.008$ & $0.816\pm0.011$ & 822 & 1491 & 2709 \\
Quad-ICNN & $0.514\pm0.083$ & $0.336\pm0.045$ & $0.204\pm0.014$ & 848 & 1592 & 3110 \\
Norm-ICNN & $0.468\pm0.026$ & $0.502\pm0.025$ & $0.582\pm0.005$ & 853 & 1602 & 3130 \\
SOC-ICNN & $\mathbf{0.357\pm0.081}$ & $\mathbf{0.243\pm0.003}$ & $\mathbf{0.158\pm0.016}$ & 527 & 1083 & 2451 \\
\bottomrule
\end{tabular}
\end{table}

\begin{table}[H]
\centering
\caption{ICKANPaperTarget function: relative $\ell_2$ error and parameter count.}
\label{tab:exp2_ickantarget}
\footnotesize
\setlength{\tabcolsep}{4pt}
\begin{tabular}{lcccccc}
\toprule
 & \multicolumn{3}{c}{RelErr (mean $\pm$ std)} & \multicolumn{3}{c}{Params} \\
\cmidrule(lr){2-4} \cmidrule(lr){5-7}
Model & $d{=}5$ & $d{=}10$ & $d{=}20$ & $d{=}5$ & $d{=}10$ & $d{=}20$ \\
\midrule
ReLU-ICNN & $0.948\pm0.002$ & $0.956\pm0.005$ & $0.963\pm0.001$ & 822 & 1491 & 2709 \\
Softplus-ICNN & $0.935\pm0.004$ & $0.950\pm0.003$ & $0.959\pm0.002$ & 822 & 1491 & 2709 \\
Quad-ICNN & $0.817\pm0.004$ & $0.752\pm0.003$ & $0.650\pm0.003$ & 848 & 1592 & 3110 \\
Norm-ICNN & $0.862\pm0.006$ & $0.881\pm0.000$ & $0.900\pm0.001$ & 853 & 1602 & 3130 \\
SOC-ICNN & $\mathbf{0.734\pm0.006}$ & $\mathbf{0.674\pm0.003}$ & $\mathbf{0.580\pm0.003}$ & 527 & 1083 & 2451 \\
\bottomrule
\end{tabular}
\end{table}

\begin{longtable}{lccccc}
\caption{Approximation results for all functions at $d=50$ (RelErr, mean $\pm$ std).}\label{tab:exp2_d50}\\
\toprule
Function & ReLU-ICNN & Softplus-ICNN & Quad-ICNN & Norm-ICNN & SOC-ICNN \\
\midrule
\endfirsthead
\multicolumn{6}{c}{{\tablename\ \thetable{} -- continued from previous page}} \\
\toprule
Function & ReLU-ICNN & Softplus-ICNN & Quad-ICNN & Norm-ICNN & SOC-ICNN \\
\midrule
\endhead
\bottomrule
\endfoot
\bottomrule
\endlastfoot
\footnotesize
Huber & $0.821\pm0.005$ & $0.815\pm0.005$ & $0.068\pm0.002$ & $0.691\pm0.003$ & $\mathbf{0.038\pm0.001}$ \\
L1Norm & $0.875\pm0.002$ & $0.870\pm0.003$ & $0.114\pm0.004$ & $0.782\pm0.004$ & $\mathbf{0.089\pm0.001}$ \\
NormEuclid & $0.345\pm0.017$ & $0.338\pm0.019$ & $0.054\pm0.001$ & $0.039\pm0.002$ & $\mathbf{0.029\pm0.001}$ \\
LogSumExpQuad & $0.585\pm0.013$ & $0.569\pm0.018$ & $0.021\pm0.000$ & $0.190\pm0.004$ & $\mathbf{0.007\pm0.000}$ \\
QuadraticIso & $0.873\pm0.006$ & $0.863\pm0.007$ & $0.102\pm0.004$ & $0.781\pm0.002$ & $\mathbf{0.077\pm0.004}$ \\
QuadraticAniso & $0.915\pm0.003$ & $0.909\pm0.004$ & $0.179\pm0.003$ & $0.856\pm0.000$ & $\mathbf{0.149\pm0.002}$ \\
NormAniso & $0.695\pm0.001$ & $0.683\pm0.004$ & $0.066\pm0.002$ & $0.431\pm0.007$ & $\mathbf{0.044\pm0.001}$ \\
Mixed & $0.861\pm0.004$ & $0.847\pm0.004$ & $0.093\pm0.002$ & $0.750\pm0.003$ & $\mathbf{0.071\pm0.003}$ \\
SoftplusSum & $0.799\pm0.007$ & $0.792\pm0.002$ & $0.115\pm0.001$ & $0.672\pm0.004$ & $\mathbf{0.097\pm0.002}$ \\
ICKANPaperTarget & $0.951\pm0.001$ & $0.947\pm0.001$ & $0.460\pm0.001$ & $0.920\pm0.001$ & $\mathbf{0.392\pm0.001}$ \\
\end{longtable}

\begin{table}[H]
\centering
\caption{Parameter counts for ICNN variants at $d=50$.}
\label{tab:exp2_d50_params}
\footnotesize
\begin{tabular}{lc}
\toprule
Model & Parameters \\
\midrule
ReLU-ICNN & 9,683 \\
Softplus-ICNN & 9,683 \\
Quad-ICNN & 9,528 \\
Norm-ICNN & 9,578 \\
SOC-ICNN & 9,423 \\
\bottomrule
\end{tabular}
\end{table}

\subsection{Full Results for Experiment 3}
\label{app:exp3_full}

\paragraph{Task construction.}
Each task instance is defined by a context parameter $\theta\in\mathbb{R}^8$, and the goal is to solve
\[
\min_{x\in\mathcal{X}} f_\theta(x).
\]
We consider three feasible sets:
\begin{itemize}
    \item \textbf{Simplex}: $\{x\in\mathbb{R}^d \mid x\ge 0,\ \sum_{i=1}^d x_i = 1\}$.
    \item \textbf{Box}: $\{x\in\mathbb{R}^d \mid 0\le x_i\le 1,\ \forall i\}$.
    \item \textbf{Budget} (capped simplex): $\{x\in\mathbb{R}^d \mid 0\le x_i\le 1,\ \sum_{i=1}^d x_i = B\}$ with $B=0.3d$.
\end{itemize}

The objective function has a common quadratic--linear backbone plus a family-specific structured term:
\[
f_\theta(x)
=
\underbrace{\frac{\alpha}{2}\sum_{i=1}^d w_i (x_i-m_i(\theta))^2 + \sum_{i=1}^d c_i(\theta)x_i}_{\text{backbone}}
+ f_{\mathrm{struct}}(x;\theta).
\]
Here $\alpha=0.35$ for the logistic and log-sum-exp families and $\alpha=1.0$ for the SOCP and Huber families; $w_i\in[0.8,1.6]$ are sampled once and fixed per task; and
\[
m(\theta)=m_0+M\theta,\qquad c(\theta)=c_0+C\theta
\]
are affine maps with random coefficients fixed within each task.

We construct six tasks from four structural families:
\begin{itemize}
    \item \textbf{SOCP family} (\texttt{simplex\_socp}, \texttt{box\_socp}):
    \[
    f_{\mathrm{struct}}(x;\theta)=\sum_{j=1}^{J}\lambda_j(\theta)\,\|A_jx-d_j(\theta)\|_2,
    \]
    with $J=1$ for the single-cone tasks and $J=2$ for the two-cone task \texttt{budget\_twocone\_socp}. The cone weights and offsets depend affinely on $\theta$ through softplus-transformed coefficients.

    \item \textbf{Logistic family} (\texttt{simplex\_logistic}):
    \[
    f_{\mathrm{struct}}(x;\theta)=\sum_{k=1}^{K}\beta_k(\theta)\log\!\bigl(1+e^{a_k^\top x-b_k(\theta)}\bigr),
    \]
    where $K=\max(6,\lfloor d/3\rfloor)$.

    \item \textbf{LogSumExp family} (\texttt{box\_logsumexp}):
    \[
    f_{\mathrm{struct}}(x;\theta)=\sum_{j=1}^{2}\beta_j(\theta)\,\operatorname{logsumexp}(A_jx-b_j(\theta)).
    \]

    \item \textbf{Huber family} (\texttt{budget\_huber}):
    \[
    f_{\mathrm{struct}}(x;\theta)=\sum_{k=1}^{K}\beta_k(\theta)\,h_{\delta}(a_k^\top x-b_k(\theta)),
    \qquad
    h_\delta(t)=
    \begin{cases}
    t^2, & |t|\le \delta,\\
    2\delta |t|-\delta^2, & |t|>\delta,
    \end{cases}
    \]
    with $\delta=0.35$ and $K=\max(8,\lfloor d/2\rfloor)$.
\end{itemize}
All random matrices and vectors are drawn once per task and dimension, then fixed across all instances.

\paragraph{Experimental protocol.}
For each task and dimension $d\in\{10,20,50\}$, we generate 1000 training instances, 1000 validation instances, and 200 test instances. Each training instance contains 64 candidate points for supervised learning of the surrogate objective. We compare five neural convex surrogates (ReLU, Softplus, Norm, Quad, and SOC) against two differentiable optimization baselines (PCF and DCP). At test time, the learned surrogate is optimized by projected gradient descent with 5 random restarts and 200 steps per restart.

\paragraph{Metrics.}
To keep the appendix tables compact, we report the number of trainable parameters (\textbf{Params}), the relative $\ell_2$ test error (\textbf{Test RelErr}), the downstream regret $f(\hat x)-f(x^\star)$ (\textbf{Regret}), the decision error $\|\hat x-x^\star\|_2$, the wall-clock training time in seconds (\textbf{Train}), and the decision-time cost in milliseconds (\textbf{Infer}). Entries are reported as mean $\pm$ standard deviation over random seeds. Neural surrogates use 3 seeds; PCF and DCP use 2--3 seeds depending on the configuration.

\paragraph{Results.}
Tables~\ref{tab:exp3_full_10}--\ref{tab:exp3_full_50} summarize the complete results. The overall strongest baseline is PCF, which achieves the lowest regret in most task--dimension settings. Within the neural surrogate family, however, the results still support the value of structured curvature injection: Quad is the most robust variant overall, while SOC remains competitive on several SOCP- and logistic-type tasks. A second consistent observation is the mismatch between upstream fitting and downstream decision quality: Softplus can attain relatively small test errors, yet often yields much worse regret than the better-performing curvature-injected models. Finally, the learned neural surrogates are substantially faster than DCP at both training and inference time, while remaining competitive with PCF in decision-time cost. The additional flexibility introduced by the Norm module does not translate into uniformly better downstream decisions, indicating that more elaborate curvature parameterizations are not always advantageous on this benchmark.

\begingroup
\footnotesize
\setlength{\tabcolsep}{4pt}
\begin{longtable}{llrrrrrr}
\caption{Experiment 3 full results at dimension $d=10$.}\label{tab:exp3_full_10}\\
\toprule
Task & Model & Params & Test RelErr & Regret & $\|x-\hat x\|_2$ & Train (s) & Infer (s) \\
\midrule
\endfirsthead
\multicolumn{8}{c}{\tablename\ \thetable{} -- continued from previous page}\\
\toprule
Task & Model & Params & Test RelErr & Regret & $\|x-\hat x\|_2$ & Train (s) & Infer (s) \\
\midrule
\endhead
\bottomrule
\endfoot
\bottomrule
\endlastfoot
box\_logsumexp & DCP & 1245 & $0.058\pm0.002$ & $0.089\pm0.005$ & $0.423\pm0.034$ & $461.3\pm0.5$ & $9.60\pm0.01$ \\
 & PCF & 753 & $0.032\pm0.001$ & $0.057\pm0.009$ & $0.288\pm0.029$ & $237.4\pm1.0$ & $1.68\pm0.00$ \\
 & Norm & 21777 & $0.078\pm0.014$ & $0.293\pm0.015$ & $0.567\pm0.025$ & $123.2\pm2.6$ & $1.30\pm0.01$ \\
 & Quad & 21586 & $0.096\pm0.026$ & $0.261\pm0.027$ & $0.517\pm0.019$ & $111.7\pm10.6$ & $1.07\pm0.00$ \\
 & ReLU & 26371 & $0.072\pm0.010$ & $0.358\pm0.023$ & $0.724\pm0.038$ & $114.0\pm0.7$ & $0.81\pm0.00$ \\
 & SOC & 21968 & $0.081\pm0.019$ & $0.299\pm0.017$ & $0.574\pm0.024$ & $84.7\pm9.7$ & $1.55\pm0.01$ \\
 & Softplus & 26371 & $0.115\pm0.001$ & $0.522\pm0.078$ & $0.979\pm0.082$ & $114.1\pm0.6$ & $0.81\pm0.01$ \\
\addlinespace
box\_socp & DCP & 1245 & $0.094\pm0.004$ & $0.225\pm0.020$ & $0.480\pm0.035$ & $463.7\pm3.3$ & $9.61\pm0.05$ \\
 & PCF & 753 & $0.062\pm0.003$ & $0.124\pm0.003$ & $0.343\pm0.010$ & $240.5\pm7.6$ & $1.68\pm0.00$ \\
 & Norm & 21777 & $0.166\pm0.056$ & $0.325\pm0.037$ & $0.569\pm0.026$ & $119.1\pm10.4$ & $1.30\pm0.01$ \\
 & Quad & 21586 & $0.155\pm0.058$ & $0.214\pm0.029$ & $0.391\pm0.025$ & $105.8\pm24.1$ & $1.08\pm0.01$ \\
 & ReLU & 26371 & $0.117\pm0.022$ & $0.328\pm0.016$ & $0.533\pm0.032$ & $115.0\pm0.1$ & $0.82\pm0.01$ \\
 & SOC & 21968 & $0.202\pm0.011$ & $0.218\pm0.015$ & $0.395\pm0.019$ & $113.0\pm20.8$ & $1.56\pm0.02$ \\
 & Softplus & 26371 & $0.171\pm0.010$ & $0.718\pm0.145$ & $0.827\pm0.102$ & $114.5\pm0.1$ & $0.82\pm0.01$ \\
\addlinespace
budget\_huber & DCP & 1245 & $0.072\pm0.007$ & $0.368\pm0.046$ & $0.807\pm0.007$ & $462.6\pm0.3$ & $12.19\pm0.02$ \\
 & PCF & 753 & $0.041\pm0.002$ & $0.139\pm0.033$ & $0.516\pm0.048$ & $235.2\pm0.1$ & $4.49\pm0.04$ \\
 & Norm & 21777 & $0.106\pm0.043$ & $0.211\pm0.022$ & $0.632\pm0.036$ & $127.1\pm0.6$ & $3.22\pm0.02$ \\
 & Quad & 21586 & $0.072\pm0.011$ & $0.140\pm0.022$ & $0.529\pm0.021$ & $116.1\pm9.2$ & $3.02\pm0.02$ \\
 & ReLU & 26371 & $0.063\pm0.004$ & $0.211\pm0.003$ & $0.631\pm0.018$ & $114.0\pm0.7$ & $2.72\pm0.03$ \\
 & SOC & 21968 & $0.089\pm0.014$ & $0.161\pm0.004$ & $0.571\pm0.006$ & $114.8\pm14.9$ & $3.33\pm0.01$ \\
 & Softplus & 26371 & $0.057\pm0.001$ & $0.501\pm0.006$ & $0.907\pm0.028$ & $114.2\pm0.5$ & $2.72\pm0.04$ \\
\addlinespace
budget\_twocone\_socp & DCP & 1245 & $0.064\pm0.007$ & $0.516\pm0.049$ & $0.674\pm0.028$ & $464.3\pm3.2$ & $12.39\pm0.24$ \\
 & PCF & 753 & $0.038\pm0.002$ & $0.217\pm0.014$ & $0.422\pm0.017$ & $234.9\pm0.3$ & $4.45\pm0.01$ \\
 & Norm & 21777 & $0.033\pm0.001$ & $0.240\pm0.037$ & $0.477\pm0.039$ & $126.7\pm1.3$ & $3.21\pm0.02$ \\
 & Quad & 21586 & $0.046\pm0.004$ & $0.228\pm0.012$ & $0.468\pm0.016$ & $110.0\pm9.6$ & $3.01\pm0.02$ \\
 & ReLU & 26371 & $0.050\pm0.005$ & $0.249\pm0.018$ & $0.483\pm0.021$ & $113.4\pm0.6$ & $2.74\pm0.03$ \\
 & SOC & 21968 & $0.035\pm0.006$ & $0.228\pm0.025$ & $0.461\pm0.025$ & $134.3\pm0.7$ & $3.33\pm0.01$ \\
 & Softplus & 26371 & $0.049\pm0.001$ & $0.676\pm0.035$ & $0.749\pm0.020$ & $109.9\pm4.6$ & $2.69\pm0.00$ \\
\addlinespace
simplex\_logistic & DCP & 1245 & $0.012\pm0.002$ & $0.334\pm0.032$ & $0.600\pm0.019$ & $462.5\pm0.2$ & $11.98\pm0.02$ \\
 & PCF & 753 & $0.004\pm0.001$ & $0.206\pm0.025$ & $0.442\pm0.029$ & $235.1\pm0.4$ & $4.26\pm0.01$ \\
 & Norm & 21777 & $0.078\pm0.025$ & $0.202\pm0.011$ & $0.444\pm0.013$ & $126.2\pm0.9$ & $3.08\pm0.03$ \\
 & Quad & 21586 & $0.069\pm0.013$ & $0.139\pm0.018$ & $0.410\pm0.016$ & $99.5\pm5.9$ & $2.89\pm0.02$ \\
 & ReLU & 26371 & $0.028\pm0.009$ & $0.171\pm0.033$ & $0.428\pm0.032$ & $114.1\pm0.6$ & $2.63\pm0.05$ \\
 & SOC & 21968 & $0.170\pm0.037$ & $0.184\pm0.013$ & $0.433\pm0.011$ & $118.2\pm10.7$ & $3.25\pm0.01$ \\
 & Softplus & 26371 & $0.013\pm0.000$ & $0.388\pm0.008$ & $0.610\pm0.009$ & $114.3\pm1.1$ & $2.61\pm0.02$ \\
\addlinespace
simplex\_socp & DCP & 1245 & $0.101\pm0.009$ & $0.170\pm0.019$ & $0.373\pm0.011$ & $464.5\pm3.7$ & $11.99\pm0.01$ \\
 & PCF & 753 & $0.065\pm0.007$ & $0.093\pm0.016$ & $0.276\pm0.028$ & $236.3\pm1.8$ & $4.28\pm0.03$ \\
 & Norm & 21777 & $0.088\pm0.031$ & $0.102\pm0.013$ & $0.276\pm0.013$ & $107.3\pm25.8$ & $3.15\pm0.03$ \\
 & Quad & 21586 & $0.084\pm0.032$ & $0.091\pm0.013$ & $0.268\pm0.020$ & $106.4\pm16.7$ & $2.90\pm0.05$ \\
 & ReLU & 26371 & $0.110\pm0.026$ & $0.086\pm0.010$ & $0.263\pm0.013$ & $114.7\pm0.7$ & $2.66\pm0.03$ \\
 & SOC & 21968 & $0.121\pm0.066$ & $0.091\pm0.010$ & $0.270\pm0.016$ & $132.8\pm3.8$ & $3.26\pm0.01$ \\
 & Softplus & 26371 & $0.085\pm0.005$ & $0.234\pm0.024$ & $0.388\pm0.008$ & $114.8\pm0.8$ & $2.62\pm0.04$ \\
\addlinespace
\end{longtable}
\endgroup

\begingroup
\footnotesize
\setlength{\tabcolsep}{4pt}
\begin{longtable}{llrrrrrr}
\caption{Experiment 3 full results at dimension $d=20$.}\label{tab:exp3_full_20}\\
\toprule
Task & Model & Params & Test RelErr & Regret & $\|x-\hat x\|_2$ & Train (s) & Infer (s) \\
\midrule
\endfirsthead
\multicolumn{8}{c}{\tablename\ \thetable{} -- continued from previous page}\\
\toprule
Task & Model & Params & Test RelErr & Regret & $\|x-\hat x\|_2$ & Train (s) & Infer (s) \\
\midrule
\endhead
\bottomrule
\endfoot
\bottomrule
\endlastfoot
box\_logsumexp & DCP & 1815 & $0.058\pm0.000$ & $0.187\pm0.001$ & $0.707\pm0.019$ & $460.7\pm0.3$ & $9.65\pm0.02$ \\
 & PCF & 1263 & $0.031\pm0.002$ & $0.094\pm0.002$ & $0.435\pm0.018$ & $234.2\pm0.5$ & $1.68\pm0.01$ \\
 & Norm & 25127 & $0.171\pm0.039$ & $0.788\pm0.030$ & $1.123\pm0.046$ & $92.4\pm8.7$ & $1.27\pm0.00$ \\
 & Quad & 24546 & $0.122\pm0.020$ & $0.614\pm0.020$ & $0.914\pm0.018$ & $76.7\pm0.9$ & $1.07\pm0.00$ \\
 & ReLU & 29261 & $0.110\pm0.015$ & $0.804\pm0.042$ & $1.204\pm0.042$ & $107.1\pm11.1$ & $0.81\pm0.00$ \\
 & SOC & 25708 & $0.121\pm0.030$ & $0.671\pm0.010$ & $0.967\pm0.023$ & $71.8\pm6.7$ & $1.52\pm0.00$ \\
 & Softplus & 29261 & $0.111\pm0.005$ & $1.033\pm0.034$ & $1.458\pm0.057$ & $114.7\pm0.6$ & $0.81\pm0.00$ \\
\addlinespace
box\_socp & DCP & 1815 & $0.095\pm0.009$ & $0.602\pm0.099$ & $0.840\pm0.060$ & $460.2\pm0.2$ & $9.62\pm0.03$ \\
 & PCF & 1263 & $0.060\pm0.001$ & $0.209\pm0.022$ & $0.489\pm0.021$ & $234.2\pm0.6$ & $1.69\pm0.01$ \\
 & Norm & 25127 & $0.099\pm0.020$ & $0.414\pm0.008$ & $0.701\pm0.010$ & $102.8\pm7.2$ & $1.28\pm0.01$ \\
 & Quad & 24546 & $0.110\pm0.024$ & $0.306\pm0.014$ & $0.571\pm0.021$ & $77.7\pm5.3$ & $1.08\pm0.01$ \\
 & ReLU & 29261 & $0.185\pm0.033$ & $0.738\pm0.132$ & $0.908\pm0.088$ & $99.3\pm23.2$ & $0.82\pm0.01$ \\
 & SOC & 25708 & $0.131\pm0.042$ & $0.385\pm0.031$ & $0.659\pm0.019$ & $74.0\pm5.1$ & $1.52\pm0.01$ \\
 & Softplus & 29261 & $0.144\pm0.007$ & $1.116\pm0.038$ & $1.147\pm0.004$ & $115.8\pm0.6$ & $0.82\pm0.01$ \\
\addlinespace
budget\_huber & DCP & 1815 & $0.084\pm0.008$ & $0.583\pm0.015$ & $1.036\pm0.043$ & $460.0\pm0.1$ & $12.19\pm0.03$ \\
 & PCF & 1263 & $0.057\pm0.004$ & $0.310\pm0.005$ & $0.793\pm0.020$ & $234.2\pm0.1$ & $4.40\pm0.00$ \\
 & Norm & 25127 & $0.669\pm0.501$ & $0.496\pm0.044$ & $1.016\pm0.025$ & $119.9\pm7.3$ & $3.19\pm0.01$ \\
 & Quad & 24546 & $0.512\pm0.174$ & $0.365\pm0.060$ & $0.901\pm0.030$ & $80.8\pm11.2$ & $3.01\pm0.01$ \\
 & ReLU & 29261 & $0.468\pm0.266$ & $0.467\pm0.038$ & $1.024\pm0.021$ & $113.3\pm2.4$ & $2.71\pm0.03$ \\
 & SOC & 25708 & $0.371\pm0.040$ & $0.391\pm0.051$ & $0.933\pm0.006$ & $108.2\pm20.9$ & $3.33\pm0.01$ \\
 & Softplus & 29261 & $0.098\pm0.004$ & $1.031\pm0.111$ & $1.303\pm0.058$ & $115.2\pm0.9$ & $2.71\pm0.03$ \\
\addlinespace
budget\_twocone\_socp & DCP & 1815 & $0.061\pm0.001$ & $0.695\pm0.120$ & $0.870\pm0.083$ & $462.1\pm0.5$ & $12.18\pm0.03$ \\
 & PCF & 1263 & $0.037\pm0.001$ & $0.389\pm0.040$ & $0.656\pm0.037$ & $237.0\pm3.6$ & $4.44\pm0.01$ \\
 & Norm & 25127 & $0.036\pm0.003$ & $0.536\pm0.043$ & $0.780\pm0.038$ & $125.0\pm2.4$ & $3.22\pm0.01$ \\
 & Quad & 24546 & $0.097\pm0.037$ & $0.403\pm0.015$ & $0.692\pm0.012$ & $91.3\pm21.8$ & $3.01\pm0.02$ \\
 & ReLU & 29261 & $0.081\pm0.016$ & $0.477\pm0.047$ & $0.748\pm0.045$ & $110.3\pm8.7$ & $2.71\pm0.01$ \\
 & SOC & 25708 & $0.061\pm0.021$ & $0.477\pm0.025$ & $0.736\pm0.027$ & $121.9\pm18.8$ & $3.33\pm0.00$ \\
 & Softplus & 29261 & $0.084\pm0.033$ & $1.199\pm0.300$ & $1.077\pm0.126$ & $85.6\pm41.7$ & $2.70\pm0.02$ \\
\addlinespace
simplex\_logistic & DCP & 1815 & $0.010\pm0.001$ & $0.467\pm0.030$ & $0.747\pm0.015$ & $460.2\pm0.2$ & $11.97\pm0.02$ \\
 & PCF & 1263 & $0.003\pm0.000$ & $0.284\pm0.029$ & $0.606\pm0.021$ & $234.9\pm0.2$ & $4.24\pm0.02$ \\
 & Norm & 25127 & $0.215\pm0.009$ & $0.349\pm0.023$ & $0.666\pm0.019$ & $93.0\pm48.6$ & $3.07\pm0.02$ \\
 & Quad & 24546 & $0.296\pm0.084$ & $0.300\pm0.042$ & $0.687\pm0.012$ & $100.5\pm16.3$ & $2.85\pm0.02$ \\
 & ReLU & 29261 & $0.139\pm0.036$ & $0.305\pm0.028$ & $0.684\pm0.015$ & $102.5\pm8.7$ & $2.66\pm0.06$ \\
 & SOC & 25708 & $0.172\pm0.050$ & $0.320\pm0.021$ & $0.671\pm0.018$ & $116.4\pm6.2$ & $3.23\pm0.01$ \\
 & Softplus & 29261 & $0.013\pm0.003$ & $0.493\pm0.031$ & $0.792\pm0.016$ & $89.3\pm6.5$ & $2.61\pm0.01$ \\
\addlinespace
simplex\_socp & DCP & 1815 & $0.086\pm0.001$ & $0.267\pm0.030$ & $0.442\pm0.021$ & $465.3\pm7.4$ & $12.01\pm0.05$ \\
 & PCF & 1263 & $0.063\pm0.006$ & $0.165\pm0.006$ & $0.372\pm0.013$ & $236.5\pm3.6$ & $4.28\pm0.05$ \\
 & Norm & 25127 & $0.329\pm0.183$ & $0.177\pm0.020$ & $0.355\pm0.021$ & $112.2\pm14.9$ & $3.10\pm0.04$ \\
 & Quad & 24546 & $0.589\pm0.133$ & $0.167\pm0.010$ & $0.349\pm0.016$ & $105.9\pm8.3$ & $2.89\pm0.04$ \\
 & ReLU & 29261 & $0.465\pm0.279$ & $0.197\pm0.016$ & $0.380\pm0.030$ & $84.4\pm17.0$ & $2.67\pm0.05$ \\
 & SOC & 25708 & $0.728\pm0.100$ & $0.181\pm0.014$ & $0.352\pm0.020$ & $118.9\pm25.5$ & $3.25\pm0.00$ \\
 & Softplus & 29261 & $0.134\pm0.017$ & $0.354\pm0.035$ & $0.488\pm0.027$ & $49.8\pm3.6$ & $2.64\pm0.05$ \\
\addlinespace
\end{longtable}
\endgroup

\begingroup
\footnotesize
\setlength{\tabcolsep}{4pt}
\begin{longtable}{llrrrrrr}
\caption{Experiment 3 full results at dimension $d=50$.}\label{tab:exp3_full_50}\\
\toprule
Task & Model & Params & Test RelErr & Regret & $\|x-\hat x\|_2$ & Train (s) & Infer (s) \\
\midrule
\endfirsthead
\multicolumn{8}{c}{\tablename\ \thetable{} -- continued from previous page}\\
\toprule
Task & Model & Params & Test RelErr & Regret & $\|x-\hat x\|_2$ & Train (s) & Infer (s) \\
\midrule
\endhead
\bottomrule
\endfoot
\bottomrule
\endlastfoot
box\_logsumexp & DCP & 3525 & $0.047\pm0.002$ & $0.679\pm0.007$ & $1.488\pm0.013$ & $460.5\pm0.1$ & $9.66\pm0.00$ \\
 & PCF & 2793 & $0.029\pm0.001$ & $0.549\pm0.002$ & $1.030\pm0.005$ & $235.8\pm0.2$ & $1.67\pm0.00$ \\
 & Norm & 101961 & $0.042\pm0.000$ & $1.583\pm0.057$ & $1.633\pm0.031$ & $80.0\pm1.7$ & $1.28\pm0.00$ \\
 & Quad & 99010 & $0.042\pm0.001$ & $1.688\pm0.067$ & $1.788\pm0.076$ & $61.5\pm6.4$ & $1.07\pm0.01$ \\
 & ReLU & 117275 & $0.058\pm0.000$ & $1.592\pm0.018$ & $1.724\pm0.028$ & $99.6\pm14.1$ & $0.81\pm0.00$ \\
 & SOC & 104912 & $0.040\pm0.000$ & $1.597\pm0.034$ & $1.601\pm0.043$ & $61.4\pm5.9$ & $1.53\pm0.00$ \\
 & Softplus & 117275 & $0.083\pm0.003$ & $2.638\pm0.067$ & $2.575\pm0.048$ & $117.5\pm0.6$ & $0.81\pm0.00$ \\
\addlinespace
box\_socp & DCP & 3525 & $0.071\pm0.001$ & $2.881\pm0.189$ & $1.813\pm0.055$ & $460.5\pm0.9$ & $9.62\pm0.00$ \\
 & PCF & 2793 & $0.041\pm0.001$ & $0.628\pm0.010$ & $0.796\pm0.017$ & $237.1\pm1.2$ & $1.68\pm0.01$ \\
 & Norm & 101961 & $0.034\pm0.001$ & $1.676\pm0.095$ & $1.243\pm0.051$ & $111.2\pm7.5$ & $1.28\pm0.00$ \\
 & Quad & 99010 & $0.039\pm0.001$ & $1.346\pm0.087$ & $1.145\pm0.044$ & $63.0\pm5.1$ & $1.08\pm0.00$ \\
 & ReLU & 117275 & $0.068\pm0.000$ & $2.032\pm0.283$ & $1.402\pm0.083$ & $94.7\pm10.0$ & $0.81\pm0.00$ \\
 & SOC & 104912 & $0.036\pm0.001$ & $1.376\pm0.019$ & $1.085\pm0.012$ & $55.4\pm10.2$ & $1.53\pm0.02$ \\
 & Softplus & 117275 & $0.078\pm0.002$ & $3.917\pm0.293$ & $2.008\pm0.067$ & $117.0\pm0.5$ & $0.82\pm0.00$ \\
\addlinespace
budget\_huber & DCP & 3525 & $0.056\pm0.002$ & $1.674\pm0.151$ & $1.880\pm0.040$ & $460.6\pm0.1$ & $12.19\pm0.04$ \\
 & PCF & 2793 & $0.038\pm0.004$ & $1.220\pm0.002$ & $1.561\pm0.033$ & $235.8\pm0.1$ & $4.47\pm0.01$ \\
 & Norm & 101961 & $0.035\pm0.000$ & $1.650\pm0.049$ & $1.833\pm0.029$ & $94.1\pm6.5$ & $3.23\pm0.01$ \\
 & Quad & 99010 & $0.035\pm0.000$ & $1.270\pm0.038$ & $1.682\pm0.027$ & $59.1\pm2.9$ & $3.05\pm0.01$ \\
 & ReLU & 117275 & $0.057\pm0.001$ & $1.665\pm0.220$ & $1.872\pm0.124$ & $117.3\pm0.6$ & $2.73\pm0.03$ \\
 & SOC & 104912 & $0.035\pm0.001$ & $1.617\pm0.010$ & $1.839\pm0.008$ & $50.8\pm8.5$ & $3.34\pm0.01$ \\
 & Softplus & 117275 & $0.060\pm0.000$ & $3.864\pm0.049$ & $2.538\pm0.012$ & $117.4\pm0.6$ & $2.76\pm0.04$ \\
\addlinespace
budget\_twocone\_socp & DCP & 3525 & $0.062\pm0.002$ & $1.815\pm0.053$ & $1.523\pm0.028$ & $462.3\pm1.2$ & $12.25\pm0.02$ \\
 & PCF & 2793 & $0.037\pm0.001$ & $0.809\pm0.094$ & $1.005\pm0.042$ & $236.0\pm0.3$ & $4.45\pm0.00$ \\
 & Norm & 101961 & $0.026\pm0.000$ & $1.254\pm0.040$ & $1.247\pm0.015$ & $80.2\pm4.8$ & $3.21\pm0.01$ \\
 & Quad & 99010 & $0.029\pm0.000$ & $1.145\pm0.072$ & $1.212\pm0.016$ & $56.0\pm6.3$ & $3.03\pm0.01$ \\
 & ReLU & 117275 & $0.061\pm0.005$ & $1.333\pm0.122$ & $1.305\pm0.062$ & $99.6\pm24.3$ & $2.73\pm0.03$ \\
 & SOC & 104912 & $0.026\pm0.001$ & $1.309\pm0.013$ & $1.282\pm0.018$ & $64.3\pm11.3$ & $3.33\pm0.01$ \\
 & Softplus & 117275 & $0.060\pm0.000$ & $3.487\pm0.179$ & $1.981\pm0.036$ & $116.2\pm2.1$ & $2.76\pm0.05$ \\
\addlinespace
simplex\_logistic & DCP & 3525 & $0.007\pm0.001$ & $0.575\pm0.003$ & $0.608\pm0.027$ & $460.4\pm0.7$ & $12.03\pm0.01$ \\
 & PCF & 2793 & $0.001\pm0.000$ & $0.394\pm0.050$ & $0.516\pm0.017$ & $236.3\pm0.2$ & $4.30\pm0.02$ \\
 & Norm & 101961 & $0.097\pm0.032$ & $0.450\pm0.020$ & $0.558\pm0.010$ & $76.1\pm42.4$ & $3.07\pm0.01$ \\
 & Quad & 99010 & $0.095\pm0.057$ & $0.363\pm0.033$ & $0.532\pm0.016$ & $55.0\pm24.6$ & $2.87\pm0.02$ \\
 & ReLU & 117275 & $0.046\pm0.018$ & $0.411\pm0.035$ & $0.560\pm0.021$ & $91.7\pm17.4$ & $2.64\pm0.00$ \\
 & SOC & 104912 & $0.082\pm0.040$ & $0.367\pm0.018$ & $0.524\pm0.005$ & $89.6\pm9.7$ & $3.24\pm0.00$ \\
 & Softplus & 117275 & $0.008\pm0.001$ & $0.604\pm0.037$ & $0.597\pm0.011$ & $98.0\pm23.8$ & $2.64\pm0.04$ \\
\addlinespace
simplex\_socp & DCP & 3525 & $0.057\pm0.003$ & $0.425\pm0.008$ & $0.490\pm0.008$ & $467.4\pm10.1$ & $12.22\pm0.29$ \\
 & PCF & 2793 & $0.028\pm0.002$ & $0.286\pm0.014$ & $0.423\pm0.014$ & $238.2\pm2.3$ & $4.31\pm0.05$ \\
 & Norm & 101961 & $0.656\pm0.174$ & $0.356\pm0.004$ & $0.430\pm0.015$ & $112.6\pm14.8$ & $3.14\pm0.01$ \\
 & Quad & 99010 & $1.174\pm0.334$ & $0.347\pm0.006$ & $0.434\pm0.008$ & $116.4\pm13.2$ & $2.94\pm0.04$ \\
 & ReLU & 117275 & $0.578\pm0.064$ & $0.365\pm0.011$ & $0.449\pm0.005$ & $109.2\pm6.2$ & $2.66\pm0.03$ \\
 & SOC & 104912 & $0.859\pm0.211$ & $0.354\pm0.004$ & $0.428\pm0.001$ & $126.1\pm22.3$ & $3.27\pm0.02$ \\
 & Softplus & 117275 & $0.105\pm0.017$ & $0.449\pm0.007$ & $0.481\pm0.024$ & $76.8\pm20.5$ & $2.64\pm0.04$ \\
\addlinespace
\end{longtable}
\endgroup

\section{Extension to Convolutional SOC-ICNNs}
\label{app:cnn_extension}

In the main text, the proposed SOC-ICNN is presented in fully connected form for clarity.
We now show that the same construction extends directly to convolutional architectures.
This extension does \emph{not} change the underlying optimization class: it only replaces dense affine maps by convolutional linear operators with parameter sharing and sparse structure.
Consequently, the convexity results, the LP value-function interpretation of the ReLU backbone in Proposition~\ref{prop:relu_lp_method}, and the SOCP value-function interpretation in Theorem~\ref{thm:soc_value_function} all remain valid.

\subsection{Convolution as a Structured Linear Operator}
\label{app:cnn_linear_operator}

Let the input be a tensor
\[
X \in \mathbb{R}^{C_0\times H_0\times W_0}.
\]
For any tensor $Y$, let $\mathrm{vec}(Y)$ denote its vectorization into a column vector.
Fix a discrete convolution operator $\mathcal{K}$ together with its stride, padding, and dilation.
Then there exists a matrix $T(\mathcal{K})$ such that
\begin{equation}
\mathrm{vec}\!\left(\mathcal{K}(Y)\right)
=
T(\mathcal{K})\,\mathrm{vec}(Y).
\label{eq:cnn_linearization_appendix}
\end{equation}
The matrix $T(\mathcal{K})$ is the usual sparse Toeplitz / block-Toeplitz matrix induced by the convolution.
Thus, every convolutional layer is still a linear map with respect to the input tensor, and every affine convolutional layer is of the form
\[
Y \mapsto \mathcal{K}(Y)+B
\]
for some bias tensor $B$.

Therefore, after vectorization, a convolutional architecture is simply a fully connected architecture with highly structured sparse matrices.
This observation is the only fact needed to transfer the value-function interpretation from the fully connected setting to the convolutional setting.

\subsection{Convolutional ReLU-ICNN Backbone}
\label{app:cnn_relu_backbone}

We define the convolutional ReLU-ICNN backbone recursively as follows.
Let $Z_0=0$, and for $\ell=1,\dots,L$ define
\begin{equation}
Z_\ell
=
\sigma\!\Bigl(
\mathcal{W}_\ell(X)
+
\mathcal{U}_\ell(Z_{\ell-1})
+
B_\ell
\Bigr),
\label{eq:cnn_relu_recursion_appendix}
\end{equation}
where:
\begin{itemize}
    \item $\mathcal{W}_\ell$ is an unconstrained input-to-hidden convolutional affine map;
    \item $\mathcal{U}_\ell$ is a hidden-to-hidden convolutional linear map whose kernel coefficients are constrained to be elementwise nonnegative;
    \item $B_\ell$ is a bias tensor;
    \item $\sigma(t)=\max\{t,0\}$ is applied elementwise.
\end{itemize}

The scalar output is defined by
\begin{equation}
f_{\mathrm{ReLU\text{-}CNN}}(X)
=
\langle C, Z_L\rangle + \langle V, X\rangle + b_0,
\label{eq:cnn_relu_output_appendix}
\end{equation}
where $C\ge 0$ elementwise, $V$ is unconstrained, and $\langle \cdot,\cdot\rangle$ denotes the Euclidean inner product between same-sized tensors.

This is the exact convolutional analogue of the fully connected ReLU-ICNN in Section~\ref{sec:method}, with dense affine maps replaced by convolutional affine maps.

\subsection{Convolutional Backbone Remains Convex}
\label{app:cnn_convexity}

We now make explicit why replacing dense layers by convolutional layers does not destroy input convexity.

\begin{proposition}[Convexity of the convolutional ReLU-ICNN backbone]
\label{prop:cnn_convexity}
The function $X \mapsto f_{\mathrm{ReLU\text{-}CNN}}(X)$ defined by \eqref{eq:cnn_relu_recursion_appendix}--\eqref{eq:cnn_relu_output_appendix} is convex with respect to the input tensor $X$.
\end{proposition}

\begin{proof}
The proof follows the same inductive logic as for standard ICNNs, but we spell it out because the convolutional case is easy to misunderstand.

\paragraph{Step 1: Convolution is affine in the input.}
For each layer $\ell$, both $\mathcal{W}_\ell(X)$ and $\mathcal{U}_\ell(Z_{\ell-1})$ are affine / linear with respect to their arguments.
Hence, if each entry of $Z_{\ell-1}(X)$ is a convex function of $X$, then each entry of $\mathcal{U}_\ell(Z_{\ell-1}(X))$ is a nonnegative linear combination of translated entries of $Z_{\ell-1}(X)$, and is therefore convex in $X$.

More explicitly, fix an output channel and spatial position index $p$.
The corresponding entry of $\mathcal{U}_\ell(Z_{\ell-1}(X))$ can be written as
\[
\bigl[\mathcal{U}_\ell(Z_{\ell-1}(X))\bigr]_p
=
\sum_{q} \sum_{r} u^{(\ell)}_{pqr}\,[Z_{\ell-1}(X)]_{q,r},
\]
where the coefficients $u^{(\ell)}_{pqr}$ are the convolution kernel coefficients.
Because $\mathcal{U}_\ell\ge 0$ elementwise, we have $u^{(\ell)}_{pqr}\ge 0$ for all indices.
Thus, if each $[Z_{\ell-1}(X)]_{q,r}$ is convex in $X$, then the above sum is convex in $X$.

\paragraph{Step 2: The pre-activation is convex.}
Define the pre-activation tensor
\[
S_\ell(X)
=
\mathcal{W}_\ell(X)+\mathcal{U}_\ell(Z_{\ell-1}(X))+B_\ell.
\]
Since $\mathcal{W}_\ell(X)$ is affine in $X$, $\mathcal{U}_\ell(Z_{\ell-1}(X))$ is convex in $X$ by Step 1, and adding a bias preserves convexity, every entry of $S_\ell(X)$ is convex in $X$.

\paragraph{Step 3: ReLU preserves convexity because it is convex and monotone nondecreasing.}
For each scalar entry $s(X)$ of $S_\ell(X)$, the corresponding activation is
\[
\sigma(s(X))=\max\{s(X),0\}.
\]
Since $s(X)$ is convex and $\sigma$ is a convex, nondecreasing scalar function, the composition $\sigma\circ s$ is convex.
Equivalently, one may directly note that $\max\{s(X),0\}$ is the maximum of two convex functions, hence convex.

Therefore, every entry of $Z_\ell(X)$ is convex in $X$.

\paragraph{Step 4: Induction over layers.}
At layer $\ell=1$, we have
\[
Z_1(X)=\sigma(\mathcal{W}_1(X)+B_1),
\]
whose entries are convex because $\mathcal{W}_1(X)+B_1$ is affine.
Applying Steps 1--3 inductively proves that every entry of every hidden tensor $Z_\ell(X)$ is convex in $X$.

\paragraph{Step 5: Final readout preserves convexity.}
The output
\[
f_{\mathrm{ReLU\text{-}CNN}}(X)
=
\langle C,Z_L(X)\rangle+\langle V,X\rangle+b_0
\]
is the sum of:
(i) a nonnegative linear combination of convex entries of $Z_L(X)$, because $C\ge 0$ elementwise;
(ii) an affine term $\langle V,X\rangle+b_0$.
Hence $f_{\mathrm{ReLU\text{-}CNN}}(X)$ is convex in $X$.
\end{proof}

The key point is therefore unchanged from the fully connected setting:
\emph{the hidden-to-hidden operator must preserve convexity, and this is ensured by constraining its coefficients to be nonnegative.}
A convolutional kernel is just a structured collection of linear coefficients; as long as those coefficients are nonnegative, the ICNN convexity argument survives verbatim.

\subsection{LP Value-Function Interpretation of the Convolutional Backbone}
\label{app:cnn_lp_backbone}

The previous proposition establishes convexity directly.
We now show that the convolutional ReLU backbone also admits the same LP value-function interpretation as Proposition~\ref{prop:relu_lp_method}.

\begin{proposition}[Convolutional LP value-function representation]
\label{prop:cnn_lp}
For every input tensor $X$, the output $f_{\mathrm{ReLU\text{-}CNN}}(X)$ equals the optimal value of the following linear program:
\begin{equation}
\begin{aligned}
f_{\mathrm{ReLU\text{-}CNN}}(X)
=
\min_{\{Z_\ell\}_{\ell=1}^{L}}
\quad
& \langle C,Z_L\rangle+\langle V,X\rangle+b_0 \\
\mathrm{s.t.}\quad
& Z_\ell \ge \mathcal{W}_\ell(X)+\mathcal{U}_\ell(Z_{\ell-1})+B_\ell,
\quad \ell=1,\dots,L, \\
& Z_\ell \ge 0,
\quad \ell=1,\dots,L, \\
& Z_0=0.
\end{aligned}
\label{eq:cnn_lp_appendix}
\end{equation}
\end{proposition}

\begin{proof}
The argument is exactly parallel to Proposition~\ref{prop:relu_lp_method}.

For fixed input $X$, the ReLU recursion in \eqref{eq:cnn_relu_recursion_appendix} computes the componentwise smallest feasible activation tensor satisfying the linear inequalities in \eqref{eq:cnn_lp_appendix}.
Because the output coefficient tensor $C$ is elementwise nonnegative, the objective is monotone nondecreasing in the entries of $Z_L$.
Hence the smallest feasible tensor minimizes the objective, and the LP optimum equals the forward output.

Equivalently, after vectorizing all tensors, \eqref{eq:cnn_lp_appendix} becomes
\begin{equation}
\begin{aligned}
\min_{\{z_\ell\}_{\ell=1}^{L}}
\quad
& c^\top z_L + v^\top x + b_0 \\
\mathrm{s.t.}\quad
& z_\ell \ge W_\ell x + U_\ell z_{\ell-1}+b_\ell,
\quad \ell=1,\dots,L, \\
& z_\ell \ge 0,
\quad \ell=1,\dots,L, \\
& z_0=0,
\end{aligned}
\label{eq:cnn_lp_vectorized_appendix}
\end{equation}
where $x=\mathrm{vec}(X)$, $z_\ell=\mathrm{vec}(Z_\ell)$, and the matrices $W_\ell,U_\ell$ are precisely the structured sparse matrices induced by the convolutional operators $\mathcal{W}_\ell,\mathcal{U}_\ell$.
This is exactly the LP form in Proposition~\ref{prop:relu_lp_method}.
\end{proof}

Thus, moving from fully connected layers to convolutional layers changes only the parameterization of the affine operators, not the underlying LP value-function nature of the ReLU-ICNN backbone.

\subsection{Convolutional SOC-ICNN}
\label{app:cnn_soc_definition}

We now augment the convolutional ReLU backbone with the same two geometric primitives used in the main text.

\paragraph{Quadratic convolutional branch.}
For each $h=1,\dots,H$, let $\mathcal{B}_h$ be an affine convolutional operator acting on $X$.
We define the $h$-th quadratic branch by
\begin{equation}
\frac{\alpha_h}{2}\left\|\mathcal{B}_h(X)+E_h\right\|_F^2,
\qquad
\alpha_h\ge 0,
\label{eq:cnn_quad_branch_appendix}
\end{equation}
where $E_h$ is a bias tensor and $\|\cdot\|_F$ denotes the Frobenius norm.

\paragraph{Conic convolutional branch.}
For each $g=1,\dots,G$, let $\mathcal{A}_g$ be another affine convolutional operator acting on $X$.
We define the $g$-th conic branch by
\begin{equation}
\lambda_g\left\|\mathcal{A}_g(X)+D_g\right\|_F,
\qquad
\lambda_g\ge 0,
\label{eq:cnn_conic_branch_appendix}
\end{equation}
where $D_g$ is a bias tensor.

The full convolutional SOC-ICNN is then defined as

\begin{equation}\label{eq:cnn_soc_output_appendix}
f_{\mathrm{SOC\text{-}CNN}}(X) =
f_{\mathrm{ReLU\text{-}CNN}}(X)
+\sum_{h=1}^{H}\frac{\alpha_h}{2}\left\|\mathcal{B}_h(X)+E_h\right\|_F^2
+\sum_{g=1}^{G}\lambda_g\left\|\mathcal{A}_g(X)+D_g\right\|_F.
\end{equation}

The role of the three terms is exactly the same as in \eqref{eq:unified_soc_icnn}: the convolutional ReLU backbone captures polyhedral structure, the quadratic convolutional branches inject explicit PSD curvature, and the conic convolutional branches inject Euclidean norm geometry at the level of feature maps.

\subsection{Convolutional SOC-ICNN Remains Convex}
\label{app:cnn_soc_convexity}

We now make explicit why the full convolutional SOC-ICNN still preserves input convexity.

\begin{proposition}[Convexity of convolutional SOC-ICNN]
\label{prop:cnn_soc_convexity}
The function $X\mapsto f_{\mathrm{SOC\text{-}CNN}}(X)$ defined in \eqref{eq:cnn_soc_output_appendix} is convex with respect to the input tensor $X$.
\end{proposition}

\begin{proof}
By Proposition~\ref{prop:cnn_convexity}, the backbone term $f_{\mathrm{ReLU\text{-}CNN}}(X)$ is convex in $X$.

For each quadratic branch, the map
\[
X \mapsto \mathcal{B}_h(X)+E_h
\]
is affine in $X$, because $\mathcal{B}_h$ is an affine convolutional operator.
Hence
\[
X \mapsto \frac{\alpha_h}{2}\|\mathcal{B}_h(X)+E_h\|_F^2
\]
is convex for every $\alpha_h\ge 0$, since the squared Frobenius norm of an affine map is convex.

Likewise, for each conic branch, the map
\[
X \mapsto \mathcal{A}_g(X)+D_g
\]
is affine in $X$, so
\[
X \mapsto \lambda_g\|\mathcal{A}_g(X)+D_g\|_F
\]
is convex for every $\lambda_g\ge 0$, since the Frobenius norm is convex and convexity is preserved under composition with affine maps.

Finally, a nonnegative sum of convex functions is convex.
Therefore,
\[
f_{\mathrm{SOC\text{-}CNN}}(X)
=
f_{\mathrm{ReLU\text{-}CNN}}(X)
+
\sum_{h=1}^{H}\frac{\alpha_h}{2}\|\mathcal{B}_h(X)+E_h\|_F^2
+
\sum_{g=1}^{G}\lambda_g\|\mathcal{A}_g(X)+D_g\|_F
\]
is convex in $X$.
\end{proof}

This proposition shows that the transition from dense affine maps to convolutional affine maps does not alter the convex-analytic foundation of SOC-ICNN.
The reason is simple but fundamental:
\emph{convexity in the main text relies only on affine dependence on the input, nonnegative hidden-to-hidden coefficients in the ICNN backbone, and convex norm / squared-norm primitives. Convolution preserves all three properties.}

\subsection{SOCP Value-Function Interpretation}
\label{app:cnn_socp}

We now state the convolutional counterpart of Theorem~\ref{thm:soc_value_function}.

For each quadratic branch, introduce an auxiliary tensor
\[
Q_h=\mathcal{B}_h(X)+E_h
\]
and an epigraph variable $s_h$ satisfying
\begin{equation}
\left(s_h,\,1,\,\mathrm{vec}(Q_h)\right)\in \mathcal{Q}_r^{n_h+2},
\label{eq:cnn_rotated_soc_appendix}
\end{equation}
where $n_h=\dim(\mathrm{vec}(Q_h))$.

For each conic branch, introduce an auxiliary tensor
\[
U_g=\mathcal{A}_g(X)+D_g
\]
and an epigraph variable $t_g$ satisfying
\begin{equation}
\left(\mathrm{vec}(U_g),\,t_g\right)\in \mathcal{Q}^{m_g+1},
\label{eq:cnn_standard_soc_appendix}
\end{equation}
where $m_g=\dim(\mathrm{vec}(U_g))$.

\begin{theorem}[Convolutional SOC-ICNN as an SOCP value function]
\label{thm:cnn_soc_value_function}
For every input tensor $X$, the output in \eqref{eq:cnn_soc_output_appendix} is exactly equal to the optimal value of the following SOCP:
\begin{equation}
\begin{aligned}
f_{\mathrm{SOC\text{-}CNN}}(X)
=
\min_{\eta}
\quad
&
\langle C,Z_L\rangle+\langle V,X\rangle+b_0
+
\sum_{h=1}^{H}\alpha_h s_h
+
\sum_{g=1}^{G}\lambda_g t_g \\
\mathrm{s.t.}\quad
&
Z_\ell \ge \mathcal{W}_\ell(X)+\mathcal{U}_\ell(Z_{\ell-1})+B_\ell,
\quad \ell=1,\dots,L, \\
&
Z_\ell \ge 0,
\quad \ell=1,\dots,L, \\
&
Z_0=0, \\
&
Q_h=\mathcal{B}_h(X)+E_h,
\quad
(s_h,1,\mathrm{vec}(Q_h))\in\mathcal{Q}_r^{n_h+2},
\quad h=1,\dots,H, \\
&
U_g=\mathcal{A}_g(X)+D_g,
\quad
(\mathrm{vec}(U_g),t_g)\in\mathcal{Q}^{m_g+1},
\quad g=1,\dots,G,
\end{aligned}
\label{eq:cnn_socp_appendix}
\end{equation}
where
\[
\eta=\{Z_\ell,Q_h,s_h,U_g,t_g\}.
\]
\end{theorem}

\begin{proof}
The proof is the exact convolutional analogue of Theorem~\ref{thm:soc_value_function}.

Let $V(X)$ denote the optimal value of \eqref{eq:cnn_socp_appendix}.

\paragraph{Upper bound: $V(X)\le f_{\mathrm{SOC\text{-}CNN}}(X)$.}
Take the forward activations $\bar Z_\ell$ produced by the convolutional ReLU backbone, so that
\[
\langle C,\bar Z_L\rangle+\langle V,X\rangle+b_0
=
f_{\mathrm{ReLU\text{-}CNN}}(X).
\]
For each quadratic branch, set
\[
\bar Q_h=\mathcal{B}_h(X)+E_h,
\qquad
\bar s_h=\frac12\|\bar Q_h\|_F^2.
\]
Then
\[
(\bar s_h,1,\mathrm{vec}(\bar Q_h))\in \mathcal{Q}_r^{n_h+2}.
\]
For each conic branch, set
\[
\bar U_g=\mathcal{A}_g(X)+D_g,
\qquad
\bar t_g=\|\bar U_g\|_F.
\]
Then
\[
(\mathrm{vec}(\bar U_g),\bar t_g)\in\mathcal{Q}^{m_g+1}.
\]
Therefore this choice is feasible for \eqref{eq:cnn_socp_appendix}, and its objective value is exactly
\[
f_{\mathrm{ReLU\text{-}CNN}}(X)
+
\sum_{h=1}^{H}\frac{\alpha_h}{2}\|\mathcal{B}_h(X)+E_h\|_F^2
+
\sum_{g=1}^{G}\lambda_g\|\mathcal{A}_g(X)+D_g\|_F
=
f_{\mathrm{SOC\text{-}CNN}}(X).
\]
Hence,
\[
V(X)\le f_{\mathrm{SOC\text{-}CNN}}(X).
\]

\paragraph{Lower bound: $V(X)\ge f_{\mathrm{SOC\text{-}CNN}}(X)$.}
Take any feasible point of \eqref{eq:cnn_socp_appendix}.
By Proposition~\ref{prop:cnn_lp},
\[
\langle C,Z_L\rangle+\langle V,X\rangle+b_0
\ge
f_{\mathrm{ReLU\text{-}CNN}}(X).
\]
Moreover, the rotated cone constraints imply
\[
s_h \ge \frac12\|Q_h\|_F^2
=
\frac12\|\mathcal{B}_h(X)+E_h\|_F^2,
\]
and the standard cone constraints imply
\[
t_g \ge \|U_g\|_F
=
\|\mathcal{A}_g(X)+D_g\|_F.
\]
Multiplying by the nonnegative coefficients $\alpha_h,\lambda_g$ and summing yields
\[
\sum_{h=1}^{H}\alpha_h s_h
\ge
\sum_{h=1}^{H}\frac{\alpha_h}{2}\|\mathcal{B}_h(X)+E_h\|_F^2,
\]
and
\[
\sum_{g=1}^{G}\lambda_g t_g
\ge
\sum_{g=1}^{G}\lambda_g\|\mathcal{A}_g(X)+D_g\|_F.
\]
Therefore every feasible objective value in \eqref{eq:cnn_socp_appendix} is at least
\[
f_{\mathrm{ReLU\text{-}CNN}}(X)
+
\sum_{h=1}^{H}\frac{\alpha_h}{2}\|\mathcal{B}_h(X)+E_h\|_F^2
+
\sum_{g=1}^{G}\lambda_g\|\mathcal{A}_g(X)+D_g\|_F
=
f_{\mathrm{SOC\text{-}CNN}}(X).
\]
Hence,
\[
V(X)\ge f_{\mathrm{SOC\text{-}CNN}}(X).
\]

Combining the two inequalities proves
\[
V(X)=f_{\mathrm{SOC\text{-}CNN}}(X).
\]
\end{proof}

\subsection{Relation to the Main Theoretical Results}
\label{app:cnn_relation_main_theory}

The convolutional extension is not a different theory, but a different parameterization of the same theory.

\begin{proposition}[Equivalence after vectorization]
\label{prop:cnn_vectorization_equivalence}
After vectorization, every convolutional SOC-ICNN can be written as a fully connected SOC-ICNN whose affine operators are structured sparse matrices induced by convolution.
Consequently, the representational and optimization-theoretic conclusions of the main text extend directly to the convolutional setting.
\end{proposition}

\begin{proof}
By \eqref{eq:cnn_linearization_appendix}, every convolutional operator admits a matrix representation after vectorization.
Applying this to all operators $\mathcal{W}_\ell,\mathcal{U}_\ell,\mathcal{B}_h,\mathcal{A}_g$ transforms the convolutional model into the fully connected form of
\eqref{eq:method_relu_lp}, \eqref{eq:unified_soc_icnn}, and \eqref{eq:soc_icnn_socp},
except that the corresponding matrices are sparse and share parameters according to convolutional structure.

Therefore:
(i) Proposition~\ref{prop:relu_lp_method} becomes Proposition~\ref{prop:cnn_lp};
(ii) Theorem~\ref{thm:soc_value_function} becomes Theorem~\ref{thm:cnn_soc_value_function};
and (iii) the function-class enlargement argument of Proposition~\ref{prop:soc_extension} carries over directly, since the quadratic and conic branches remain exact structured primitives after vectorization.
\end{proof}

In particular, the strict enlargement from polyhedral LP geometry to conic SOCP geometry is unchanged.
What changes is only the inductive bias:
the convolutional version is tailored to spatially structured inputs and benefits from local connectivity and parameter sharing, while retaining the same convexity guarantees and the same exact value-function interpretation.

\begin{remark}[On allowed CNN modules]
\label{rem:cnn_modules}
The above extension applies directly to any \emph{linear} convolutional module, including standard convolutions, strided convolutions, dilated convolutions, transposed convolutions, $1\times 1$ convolutions, and average pooling.
To preserve the exact LP/SOCP value-function interpretation in the present form, one should avoid inserting additional nonlinear operators such as  batch normalization inside the convex backbone, unless they are analyzed separately.
\end{remark}

\subsection{Convolutional Downstream Test on Real Images}
\label{app:cnn_downstream_real}

To complement the synthetic approximation and downstream decision experiments in the main text, we include a small real-image downstream test for the convolutional SOC-ICNN extension. 
The purpose of this experiment is not to compete with specialized image restoration methods, but to verify that the same LP-to-SOCP lifting principle remains operationally useful when the decision variable has real spatial structure. 
The convolutional ReLU baseline follows the input-convex construction principle of ICNNs \citep{AmosXuKolter2017ICNN}, with nonnegative hidden-to-hidden convolutional weights to preserve convexity with respect to the image decision variable.

\paragraph{Task.}
We use the Olivetti Faces dataset, which consists of real grayscale face images at native \(64\times64\) resolution. 
For each clean image \(X^{\natural}\), we generate a corrupted observation \(Y\) by applying a mild Gaussian blur, additive Gaussian noise, a light random pixel mask, and several rectangular missing blocks. 
The decision variable is the restored image \(X\in[0,1]^{1\times64\times64}\). 
The true convex restoration energy is
\[
\begin{aligned}
F_{\mathrm{img}}(X;Y,M)
=
&\frac{\rho}{2|\Omega|}
\|M\odot(H*X-Y)\|_F^2
+
\frac{\beta}{2|\Omega|}
\|L*X\|_F^2 \\
&+
\frac{\lambda}{|\Omega|}
\sum_{(i,j)\in\Omega}
\sqrt{(D_xX)_{i,j}^2+(D_yX)_{i,j}^2+\varepsilon},
\end{aligned}
\]
where \(M\) is the observation mask, \(\Omega\) is the image lattice, \(H\) is a Gaussian blur operator, \(L\) is the discrete Laplacian filter, and \(D_x,D_y\) are finite-difference filters. 
This objective combines masked data fidelity, convolutional quadratic smoothing, and a TV-like conic regularizer, making it a natural downstream test for convolutional SOC-ICNNs.

\paragraph{Models.}
We compare five convolutional convex surrogates: Conv-ReLU, Conv-Softplus, Conv-Quad, Conv-Norm, and Conv-SOC. 
All models use the same three-layer convolutional ICNN backbone with \(10\) hidden channels. 
The Conv-Quad and Conv-Norm variants add only the quadratic or conic branch, respectively, while Conv-SOC includes both branches. 
The quadratic branch uses \(16\) convolutional output channels, and the conic branch uses \(8\) two-dimensional norm groups. 
All hidden-to-hidden convolutional kernels in the ICNN backbone are constrained to be elementwise nonnegative.

\paragraph{Data and optimization protocol.}
We split the \(400\) Olivetti images into \(280\) training images, \(60\) validation images, and \(60\) test images. 
For each image, we generate \(16\) candidate restored images and evaluate the true convex energy to train the surrogate by mean-squared error regression. 
The candidate set includes corrupted observations, clean images, blurred images, random images, noisy perturbations, and convex mixtures. 
All models are trained for \(150\) epochs using AdamW with learning rate \(2\times10^{-3}\), batch size \(96\), and weight decay \(10^{-6}\). 
For downstream evaluation, each learned surrogate is optimized over \(X\in[0,1]^{1\times64\times64}\) using projected Adam for \(350\) steps and \(3\) random restarts. 
The reference solution \(X^\star\) is computed by approximately minimizing the true convex energy using projected Adam for \(700\) steps and the same box projection. 
We report relative prediction error, true-energy regret \(F_{\mathrm{img}}(\hat X;Y,M)-F_{\mathrm{img}}(X^\star;Y,M)\), pixel-level decision error \(\|\hat X-X^\star\|_F/\sqrt{|\Omega|}\), and downstream inference time.

\paragraph{Implementation.}
The experiment is implemented in Python~3.11 with PyTorch~2.0 and run on a single NVIDIA RTX 4060 Ti GPU. 
All random seeds are fixed to \(0\). 
The corruption parameters are \(\rho=1.0\), \(\beta=0.035\), \(\lambda=0.055\), \(\varepsilon=10^{-4}\), Gaussian blur standard deviation \(0.7\), Gaussian noise standard deviation \(0.025\), base pixel keep probability \(0.90\), and \(2\)--\(4\) rectangular holes per image with side lengths between \(8\) and \(18\) pixels.

\begin{table}[t]
\centering
\caption{Convolutional downstream image restoration test on Olivetti Faces \(64\times64\) grayscale images. Regret and decision error are evaluated using the true convex restoration energy after optimizing each learned convolutional surrogate. Lower is better for RelErr, regret, and decision error.}
\label{tab:cnn_downstream_real}
\small
\setlength{\tabcolsep}{4pt}
\renewcommand{\arraystretch}{1.08}
\begin{tabular}{lcccc}
\toprule
Model & RelErr & Regret & Decision error & Infer. (ms) \\
\midrule
Conv-ReLU & \(0.0492\) & \(0.0254\pm0.0033\) & \(0.2497\pm0.0174\) & \(86.73\) \\
Conv-Softplus & \(0.4434\) & \(0.0821\pm0.0192\) & \(0.4736\pm0.0527\) & \(93.09\) \\
Conv-Quad & \(0.0474\) & \(0.0111\pm0.0014\) & \(0.1633\pm0.0107\) & \(103.72\) \\
Conv-Norm & \(0.0473\) & \(0.0074\pm0.0026\) & \(0.1377\pm0.0282\) & \(66.77\) \\
Conv-SOC & \(\mathbf{0.0461}\) & \(\mathbf{0.0041}\pm0.0011\) & \(\mathbf{0.1303}\pm0.0207\) & \(84.84\) \\
\bottomrule
\end{tabular}
\end{table}

\begin{figure}[t]
\centering
\includegraphics[width=\linewidth]{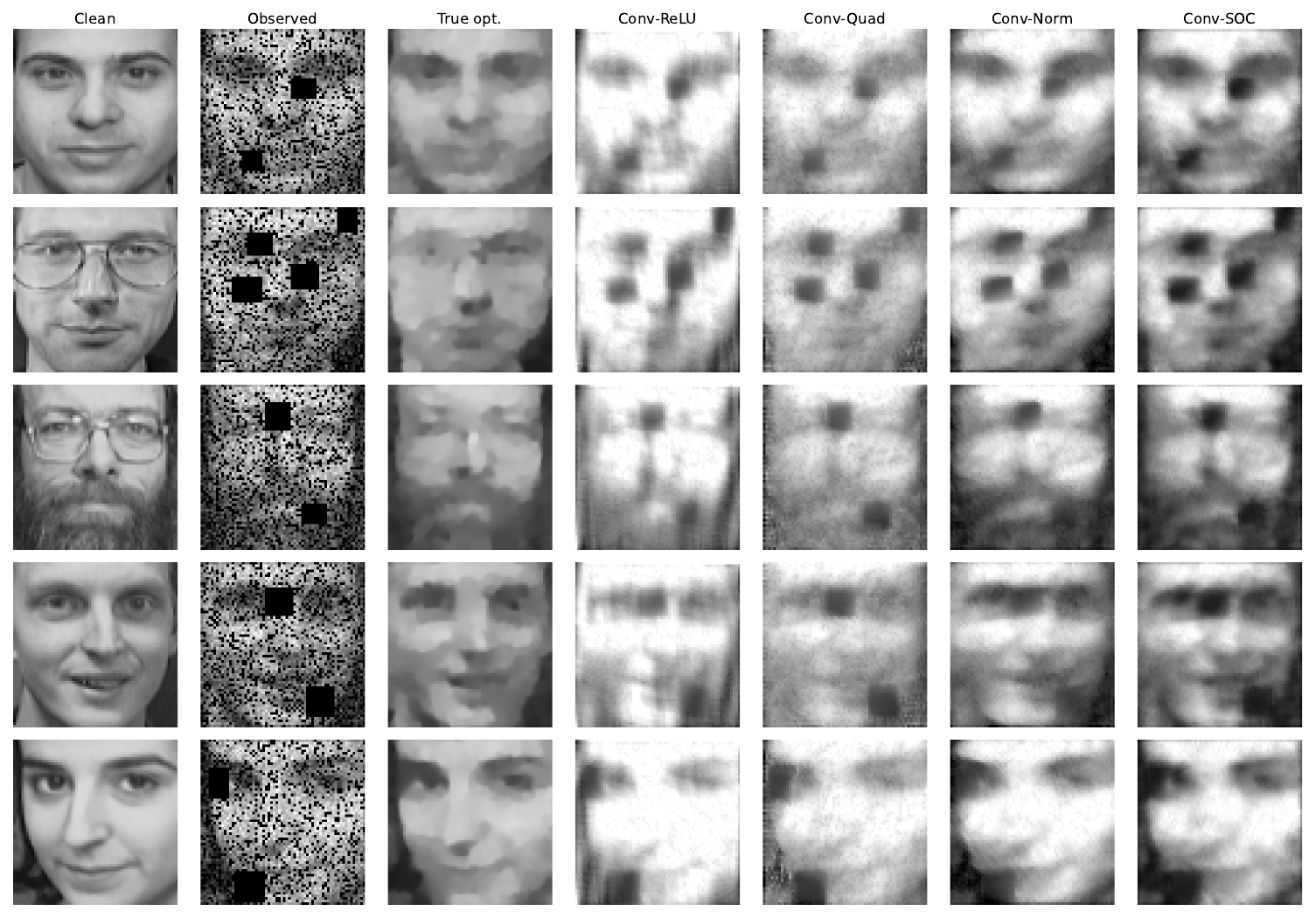}
\caption{Qualitative downstream image restoration examples on Olivetti Faces \(64\times64\) grayscale images. Each row is a test instance, and each learned convolutional surrogate is optimized over the restored image. The structured Quad, Norm, and SOC variants recover the dominant facial structure under the learned convex energy; Conv-SOC gives the lowest relative prediction error, true-energy regret, and pixel-level decision error in Table~\ref{tab:cnn_downstream_real}.}
\label{fig:cnn_downstream_real}
\end{figure}

\section{Extension to Recurrent SOC-ICNNs}
\label{app:rnn_extension}

In addition to convolutional architectures, the same SOC-ICNN construction also extends to finite-horizon recurrent architectures. 
The key point is that a recurrent network over a fixed horizon can be unrolled into a feed-forward computation graph with shared or time-varying affine maps. 
Once unrolled, the same monotonicity and epigraph arguments used in the main text apply directly. 
Thus, the recurrent extension does not change the underlying value-function class: it replaces a dense feed-forward backbone by a structured affine recurrence while preserving input convexity and the SOCP value-function interpretation.

\subsection{Finite-Horizon Recurrent Backbone}
\label{app:rnn_backbone}

Let the input be a sequence
\[
\bm{x}_{1:T}:=(\bm{x}_1,\dots,\bm{x}_T),
\qquad
\bm{x}_t\in\mathbb{R}^{d_t}.
\]
A finite-horizon recurrent ReLU-ICNN backbone is defined by
\begin{equation}
\bm{z}_t
=
\sigma\!\left(
W_t\bm{x}_t+R_t\bm{z}_{t-1}+\bm{b}_t
\right),
\qquad
t=1,\dots,T,
\qquad
\bm{z}_0=\bm{0},
\label{eq:rnn_relu_backbone_appendix}
\end{equation}
where \(\sigma(a)=\max\{a,0\}\) is applied elementwise. 
The recurrent hidden-to-hidden weights satisfy
\[
R_t\ge0,
\qquad
\bm{c}\ge0,
\]
elementwise. 
The terminal-output recurrent ICNN is
\begin{equation}
f_{\mathrm{ReLU\text{-}RNN}}(\bm{x}_{1:T})
=
\bm{c}^{\top}\bm{z}_T
+
\sum_{t=1}^{T}\bm{v}_t^\top\bm{x}_t
+
b_0 .
\label{eq:rnn_relu_output_appendix}
\end{equation}
The matrices \(W_t\), vectors \(\bm v_t\), and biases \(\bm b_t\) are unconstrained. 
The time-varying notation is used for generality; the usual weight-sharing RNN is obtained by setting \(W_t=W\), \(R_t=R\), and \(\bm b_t=\bm b\).

\subsection{LP Value-Function Interpretation of the Recurrent Backbone}
\label{app:rnn_lp}

For any fixed sequence \(\bm{x}_{1:T}\), the recurrent backbone admits an LP value-function representation:
\begin{equation}
\begin{aligned}
f_{\mathrm{ReLU\text{-}RNN}}(\bm{x}_{1:T})
=
\min_{\{\bm z_t\}_{t=1}^{T}}\quad
&\bm{c}^{\top}\bm{z}_T
+\sum_{t=1}^{T}\bm{v}_t^\top\bm{x}_t
+b_0\\
{\rm s.t.}\quad
&\bm{z}_t
\ge
W_t\bm{x}_t+R_t\bm{z}_{t-1}+\bm{b}_t,\quad
\bm{z}_t\ge\bm0,\quad t=1,\dots,T,\\
&\bm z_0=\bm0 .
\end{aligned}
\label{eq:rnn_relu_lp_appendix}
\end{equation}
Indeed, the ReLU recurrence in \eqref{eq:rnn_relu_backbone_appendix} constructs the componentwise minimal feasible hidden sequence. 
Because \(R_t\ge0\), monotonicity propagates through time, and because \(\bm c\ge0\), the terminal objective is minimized by the forward recurrent activations. 
Thus, the recurrent ReLU-ICNN is still an LP value-function model after unrolling.

\subsection{Recurrent SOC-ICNN Architecture}
\label{app:rnn_soc_architecture}

We now add the same quadratic and conic structural branches used in the main SOC-ICNN. 
For a sequence input, the branch maps are affine functions of the entire sequence:
\[
\mathcal{B}_h(\bm{x}_{1:T})
=
\sum_{t=1}^{T}B_{h,t}\bm{x}_t+\bm e_h,
\qquad
\mathcal{A}_g(\bm{x}_{1:T})
=
\sum_{t=1}^{T}A_{g,t}\bm{x}_t+\bm d_g.
\]
The recurrent SOC-ICNN is then defined as
\begin{equation}
\begin{aligned}
f_{\mathrm{SOC\text{-}RNN}}(\bm{x}_{1:T})
=
&f_{\mathrm{ReLU\text{-}RNN}}(\bm{x}_{1:T})
+
\sum_{h=1}^{H}
\frac{\alpha_h}{2}
\left\|
\mathcal{B}_h(\bm{x}_{1:T})
\right\|_2^2\\
&+
\sum_{g=1}^{G}
\lambda_g
\left\|
\mathcal{A}_g(\bm{x}_{1:T})
\right\|_2 ,
\end{aligned}
\label{eq:rnn_soc_output_appendix}
\end{equation}
where \(\alpha_h\ge0\) and \(\lambda_g\ge0\). 
This is the direct recurrent analogue of \eqref{eq:unified_soc_icnn}: the ReLU recurrent backbone captures polyhedral temporal structure, while the quadratic and conic branches inject PSD curvature and Euclidean norm geometry over the full input sequence.

\subsection{Recurrent SOC-ICNN Remains Convex}
\label{app:rnn_soc_convexity}

\begin{proposition}[Convexity of recurrent SOC-ICNN]
\label{prop:rnn_soc_convexity}
For any finite horizon \(T\), the function
\(\bm{x}_{1:T}\mapsto f_{\mathrm{SOC\text{-}RNN}}(\bm{x}_{1:T})\)
defined in \eqref{eq:rnn_soc_output_appendix} is convex with respect to the full input sequence.
\end{proposition}

\begin{proof}
We first prove that the recurrent ReLU backbone is convex in \(\bm{x}_{1:T}\). 
The initial state \(\bm z_0=\bm0\) is constant and hence convex. 
Assume that each component of \(\bm z_{t-1}\) is convex in \(\bm{x}_{1:t-1}\). 
Since \(R_t\ge0\), the term \(R_t\bm z_{t-1}\) is a nonnegative linear combination of convex functions. 
The term \(W_t\bm x_t+\bm b_t\) is affine in \(\bm x_t\). 
Therefore each component of
\[
W_t\bm x_t+R_t\bm z_{t-1}+\bm b_t
\]
is convex in \(\bm{x}_{1:t}\). 
Because ReLU is convex and nondecreasing, applying \(\sigma\) elementwise preserves convexity. 
Thus each component of \(\bm z_t\) is convex by induction. 
Since \(\bm c\ge0\), the terminal term \(\bm c^\top\bm z_T\) is convex, and the passthrough term \(\sum_{t=1}^T\bm v_t^\top\bm x_t+b_0\) is affine. 
Hence \(f_{\mathrm{ReLU\text{-}RNN}}\) is convex.

For each quadratic branch, \(\mathcal B_h(\bm x_{1:T})\) is affine in the full sequence, so
\[
\bm x_{1:T}
\mapsto
\frac{\alpha_h}{2}
\left\|
\mathcal B_h(\bm x_{1:T})
\right\|_2^2
\]
is convex for every \(\alpha_h\ge0\). 
Likewise, each conic branch
\[
\bm x_{1:T}
\mapsto
\lambda_g
\left\|
\mathcal A_g(\bm x_{1:T})
\right\|_2
\]
is convex for every \(\lambda_g\ge0\). 
A nonnegative sum of convex functions is convex, proving the claim.
\end{proof}

\subsection{SOCP Value-Function Interpretation}
\label{app:rnn_socp}

We now state the recurrent counterpart of Theorem~\ref{thm:soc_value_function}. 
For each quadratic branch, introduce
\[
\bm q_h=\mathcal B_h(\bm x_{1:T})
\]
and an epigraph variable \(s_h\) satisfying
\[
(s_h,1,\bm q_h)\in\mathcal Q_r^{m_h+2}.
\]
For each conic branch, introduce
\[
\bm u_g=\mathcal A_g(\bm x_{1:T})
\]
and an epigraph variable \(t_g\) satisfying
\[
(\bm u_g,t_g)\in\mathcal Q^{k_g+1}.
\]
Here \(m_h=\dim(\bm q_h)\) and \(k_g=\dim(\bm u_g)\).

\begin{proposition}[Recurrent SOC-ICNN as an SOCP value function]
\label{prop:rnn_socp}
For every finite horizon \(T\) and every sequence input \(\bm{x}_{1:T}\), the recurrent SOC-ICNN in \eqref{eq:rnn_soc_output_appendix} is equal to the optimal value of the following SOCP:
\begin{equation}
\begin{aligned}
f_{\mathrm{SOC\text{-}RNN}}(\bm{x}_{1:T})
=
\min_{\eta}\quad
&\bm c^\top\bm z_T
+\sum_{t=1}^{T}\bm v_t^\top\bm x_t
+b_0
+\sum_{h=1}^{H}\alpha_h s_h
+\sum_{g=1}^{G}\lambda_g t_g\\
{\rm s.t.}\quad
&\bm z_t\ge W_t\bm x_t+R_t\bm z_{t-1}+\bm b_t,\quad
\bm z_t\ge\bm0,\quad t=1,\dots,T,\\
&\bm z_0=\bm0,\\
&\bm q_h=\sum_{t=1}^{T}B_{h,t}\bm x_t+\bm e_h,\quad
(s_h,1,\bm q_h)\in\mathcal Q_r^{m_h+2},\quad h=1,\dots,H,\\
&\bm u_g=\sum_{t=1}^{T}A_{g,t}\bm x_t+\bm d_g,\quad
(\bm u_g,t_g)\in\mathcal Q^{k_g+1},\quad g=1,\dots,G,
\end{aligned}
\label{eq:rnn_socp_appendix}
\end{equation}
where
\[
\eta=\{\bm z_t,\bm q_h,s_h,\bm u_g,t_g:\ t=1,\dots,T,\ h=1,\dots,H,\ g=1,\dots,G\}.
\]
\end{proposition}

\begin{proof}
The proof is identical in structure to Theorem~\ref{thm:soc_value_function}. 
Let \(V(\bm x_{1:T})\) denote the optimal value of \eqref{eq:rnn_socp_appendix}. 
For the upper bound, choose the recurrent forward activations \(\bar{\bm z}_t\), set
\[
\bar{\bm q}_h=\sum_{t=1}^{T}B_{h,t}\bm x_t+\bm e_h,
\qquad
\bar s_h=\frac12\|\bar{\bm q}_h\|_2^2,
\]
and
\[
\bar{\bm u}_g=\sum_{t=1}^{T}A_{g,t}\bm x_t+\bm d_g,
\qquad
\bar t_g=\|\bar{\bm u}_g\|_2.
\]
This point is feasible for \eqref{eq:rnn_socp_appendix} and attains exactly the value in \eqref{eq:rnn_soc_output_appendix}. 
Hence \(V(\bm x_{1:T})\le f_{\mathrm{SOC\text{-}RNN}}(\bm x_{1:T})\).

For the lower bound, any feasible point satisfies the LP lower bound from \eqref{eq:rnn_relu_lp_appendix}. 
The rotated-cone constraints imply
\[
s_h\ge\frac12\left\|\sum_{t=1}^{T}B_{h,t}\bm x_t+\bm e_h\right\|_2^2,
\]
and the standard-cone constraints imply
\[
t_g\ge
\left\|\sum_{t=1}^{T}A_{g,t}\bm x_t+\bm d_g\right\|_2.
\]
Multiplying by \(\alpha_h,\lambda_g\ge0\) and summing gives
\(V(\bm x_{1:T})\ge f_{\mathrm{SOC\text{-}RNN}}(\bm x_{1:T})\). 
Combining both inequalities proves equality.
\end{proof}

This proposition shows that the recurrent extension preserves the same mathematical backbone as the fully connected and convolutional SOC-ICNNs: after finite-horizon unrolling, the ReLU recurrent part is an LP value function, and the quadratic/conic branches lift it to an SOCP value function. 
The extension is therefore architectural rather than conceptual; it adds temporal parameter sharing and recurrent structure without leaving the LP-to-SOCP value-function framework.

\subsection{Recurrent Downstream Test on Trajectory Smoothing}
\label{app:rnn_downstream}

To further test the finite-horizon recurrent SOC-ICNN extension, we consider a downstream trajectory-smoothing problem. 
The goal of this experiment is to verify that the recurrent LP-to-SOCP lifting remains useful when the decision variable is a temporal sequence rather than a static vector or image. 
The recurrent ReLU baseline follows the input-convex construction principle of ICNNs \citep{AmosXuKolter2017ICNN}, with nonnegative hidden-to-hidden recurrent weights to preserve convexity with respect to the decision sequence.

\paragraph{Task.}
For each reference trajectory \(R=(r_1,\dots,r_T)\), the decision variable is a smoothed trajectory \(U=(u_1,\dots,u_T)\), where \(u_t\in\mathbb R^p\). 
We use \(T=40\) and \(p=4\). 
The true convex sequence energy is
\[
\begin{aligned}
F_{\mathrm{seq}}(U;R)
=
&\frac{\rho}{2T}\sum_{t=1}^{T}\|u_t-r_t\|_2^2
+
\frac{\beta}{2(T-2)}
\sum_{t=2}^{T-1}
\|u_{t+1}-2u_t+u_{t-1}\|_2^2 \\
&+
\frac{\lambda}{T-1}
\sum_{t=2}^{T}
\sqrt{\|u_t-u_{t-1}\|_2^2+\varepsilon}
+
\frac{\gamma}{2T}
\sum_{t=1}^{T}\|u_t\|_2^2 ,
\end{aligned}
\]
with box constraints \(u_t\in[-1,1]^p\). 
The first term tracks the reference sequence, the second term penalizes acceleration, the third term is a TV-like conic switching cost, and the last term is a weak magnitude regularizer. 
This objective naturally matches the SOC-RNN architecture: the quadratic branch models temporal curvature, while the conic branch models switching costs.

\paragraph{Models.}
We compare five recurrent convex surrogates: RNN-ReLU, RNN-Softplus, RNN-Quad, RNN-Norm, and RNN-SOC. 
All models use the same recurrent ICNN backbone with hidden dimension \(20\). 
The RNN-Quad and RNN-Norm variants add only the temporal quadratic or temporal conic branch, respectively, while RNN-SOC includes both branches. 
The quadratic branch uses \(24\) temporal convolutional channels, and the conic branch uses \(12\) two-dimensional norm groups. 
The hidden-to-hidden recurrent weights in the ICNN backbone are constrained to be elementwise nonnegative.

\paragraph{Data and optimization protocol.}
We generate \(1200\) training, \(200\) validation, and \(200\) test reference trajectories. 
Each reference trajectory combines sinusoidal components, low-frequency drift, and several piecewise offsets. 
For each trajectory, we generate \(14\) candidate decision sequences and evaluate the true convex energy to train the surrogate by mean-squared error regression. 
All models are trained for \(150\) epochs using AdamW with learning rate \(2\times10^{-3}\), batch size \(256\), and weight decay \(10^{-6}\). 
For downstream evaluation, each learned surrogate is optimized over \(U\in[-1,1]^{T\times p}\) using projected Adam for \(400\) steps and \(4\) random restarts. 
The reference solution \(U^\star\) is computed by approximately minimizing the true convex sequence energy using projected Adam for \(900\) steps and the same box projection. 
We report relative prediction error, true-energy regret \(F_{\mathrm{seq}}(\hat U;R)-F_{\mathrm{seq}}(U^\star;R)\), decision error \(\|\hat U-U^\star\|_F/\sqrt{Tp}\), and downstream inference time.

\paragraph{Implementation.}
The experiment is implemented in Python~3.11 with PyTorch~2.0 and run on a single NVIDIA RTX 4060 Ti GPU. 
All random seeds are fixed to \(0\). 
The energy parameters are \(\rho=1.0\), \(\beta=0.20\), \(\lambda=0.18\), \(\gamma=0.02\), and \(\varepsilon=10^{-5}\).

\begin{table}[t]
\centering
\caption{Recurrent downstream trajectory-smoothing test. Regret and decision error are evaluated using the true convex sequence energy after optimizing each learned recurrent surrogate. Lower is better for RelErr, regret, and decision error.}
\label{tab:rnn_downstream}
\small
\setlength{\tabcolsep}{4pt}
\renewcommand{\arraystretch}{1.08}
\begin{tabular}{lcccc}
\toprule
Model  & RelErr & Regret & Decision error & Infer. (ms) \\
\midrule
RNN-ReLU  & \(0.1920\) & \(0.4515\pm0.0664\) & \(0.5458\pm0.0320\) & \(405.98\) \\
RNN-Softplus  & \(0.2359\) & \(0.4150\pm0.0676\) & \(0.5528\pm0.0335\) & \(547.33\) \\
RNN-Quad  & \(0.0914\) & \(0.0697\pm0.0152\) & \(0.3094\pm0.0433\) & \(381.80\) \\
RNN-Norm  & \(0.1154\) & \(0.0903\pm0.0177\) & \(0.3708\pm0.0407\) & \(554.24\) \\
RNN-SOC  & \(\mathbf{0.0671}\) & \(\mathbf{0.0524}\pm0.0156\) & \(\mathbf{0.2810}\pm0.0489\) & \(431.76\) \\
\bottomrule
\end{tabular}
\end{table}

\begin{figure}[t]
\centering
\includegraphics[width=\linewidth]{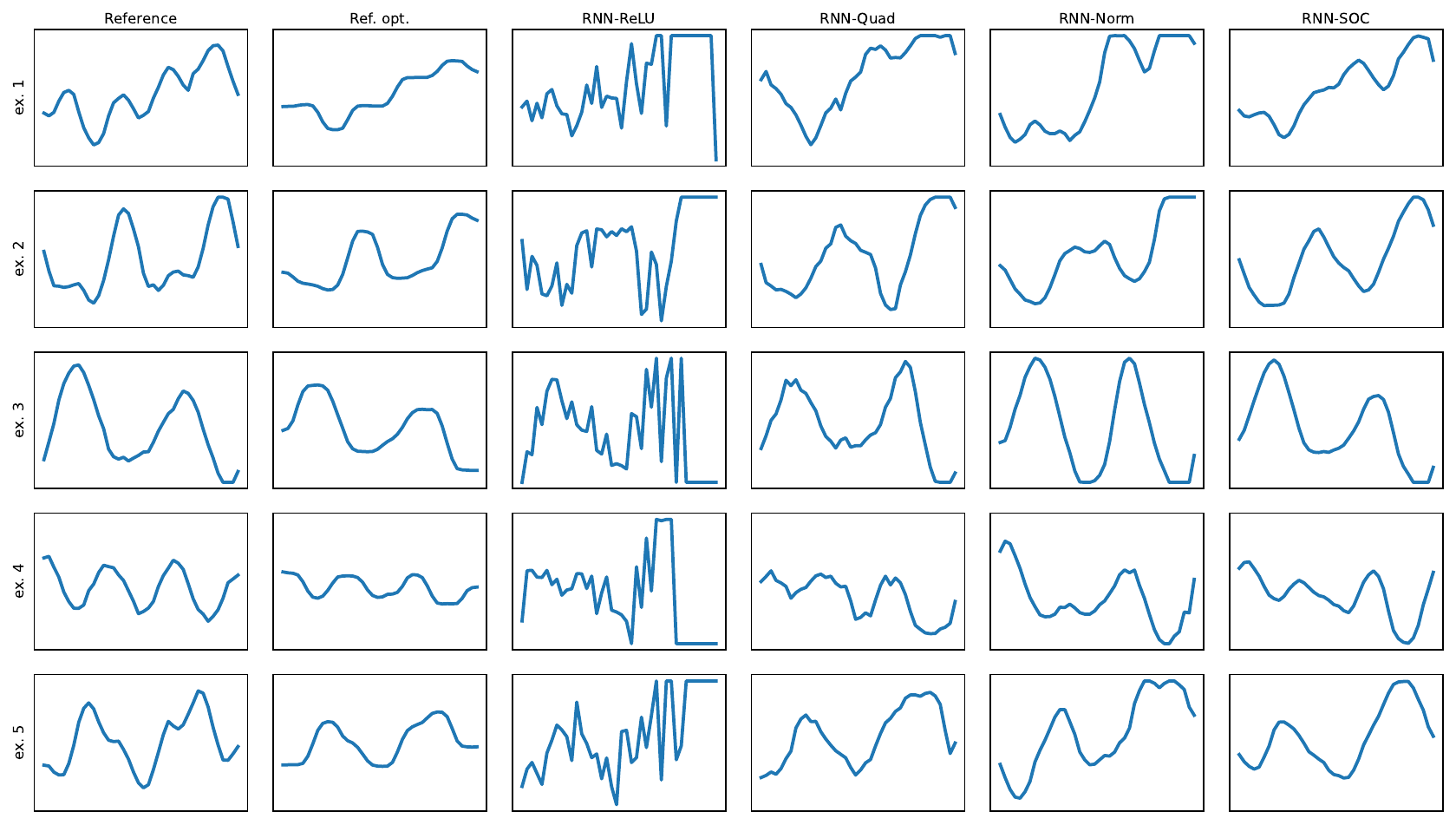}
\caption{Qualitative downstream trajectory-smoothing examples. Each row is a test instance, and each column shows one trajectory component of the optimized decision. RNN-ReLU produces oscillatory and frequently clipped trajectories, while the structured Quad, Norm, and SOC variants better recover the reference optimum. RNN-SOC gives the lowest relative prediction error, true-energy regret, and decision error in Table~\ref{tab:rnn_downstream}.}
\label{fig:rnn_downstream}
\end{figure}


\end{document}